\documentclass{article}

\usepackage{microtype}
\usepackage{graphicx}
\usepackage{subfigure}
\usepackage{booktabs} 

\usepackage{hyperref}

\usepackage[accepted]{icml2021}

\icmltitlerunning{Query Complexity of Adversarial Attacks}

\usepackage{amsmath}

\usepackage{hyperref}
\usepackage{url}

\usepackage{amsthm} 
\usepackage{xcolor}
\usepackage{amssymb}
\usepackage{verbatim}
\usepackage{tablefootnote}
\usepackage{dsfont}

\usepackage{todonotes}

\usepackage[utf8]{inputenc}
\usepackage{thmtools,thm-restate}

\usepackage{thmtools}
\usepackage{mathtools}
\usepackage{enumerate}
 
\newcommand{\N}{\mathbb{N}}

\newcommand{\R}{\mathbb{R}}

\newtheorem{theorem}{Theorem}
\newtheorem{lemma}{Lemma}
\newtheorem{observation}{Observation}
\newtheorem{corollary}{Corollary}
\newtheorem{conjecture}{Conjecture}

\newtheorem{remark}{Remark}

\newtheorem{definition}{Definition}

\newcommand{\e}{\epsilon}
\newcommand{\adv}{\mathcal{A}}

\newcounter{mycomment}
\newcommand{\comm}[2]{%
\refstepcounter{mycomment}
{%
    \todo[author = \textbf{#1~\#~\themycomment}, color={red!100!green!35}, fancyline, size = \footnotesize]{%
        #2}%
    }
}

\usepackage{booktabs}

\begin{document}

\twocolumn[
\icmltitle{Query Complexity of Adversarial Attacks}




\begin{icmlauthorlist}
\icmlauthor{Grzegorz Głuch}{to}
\icmlauthor{Rüdiger Urbanke}{to}
\end{icmlauthorlist}

\icmlaffiliation{to}{School of Computer and Communication Sciences, EPFL, Switzerland}

\icmlcorrespondingauthor{Grzegorz Głuch}{grzegorz.gluch@epfl.ch}

\icmlkeywords{Machine Learning, ICML}

\vskip 0.3in
]



\printAffiliationsAndNotice{\icmlEqualContribution} 

\begin{abstract}
There are two main attack models considered in the adversarial robustness literature: black-box and white-box. We consider these threat models as two ends of a fine-grained spectrum, indexed by the number of queries the adversary can ask. Using this point of view we investigate how many queries the adversary needs to make to design an attack that is comparable to the best possible attack in the white-box model. We give a lower bound on that number of queries in terms of entropy of decision boundaries of the classifier. Using this result we analyze two classical learning algorithms on two synthetic tasks for which we prove meaningful security guarantees. The obtained bounds suggest that some learning algorithms are inherently more robust against query-bounded adversaries than others.

\end{abstract}

\section{Introduction}

Modern neural networks achieve high accuracy on tasks such as image classification \citep{howtobecomefamous} or speech recognition \citep{speachrecognition}. However, they are typically susceptible to small, adversarially-chosen perturbations of the inputs \citep{intriguingprop,neuralnetseasilyfooled}: more precisely, given a correctly-classified input $x$, one can typically find a small perturbation $\delta$ such that $x + \delta$ is misclassified by the network while to the human eye this perturbation is not perceptible.

There are two main threat models considered in the literature: black-box and white-box. In the white-box model, on the one hand, the attacker \citep{evassionattackontesttime,bestwhiteboxmadrychallenge} is assumed to have access to a full description of the model. For the case of neural networks that amounts to a knowledge of the architecture and the weights. In the black-box model, on the other hand, the adversary \citep{goodfellowblackbox,zerothblackbox,transferableblackbox,blackboxfirstmadrychallenge,adversarialserviceblackbox} can only observe the input-output behavior of the model. Many defenses have been proposed to date. To mention just some -- adversarial learning \citep{goodfellow2014explaining,madry2018towards,tramer2018ensemble,xiao2018training}, input denoising \citep{dimensionalityreduction,featuresq}, or more recently, randomized smoothing \citep{Lacuyer19,smoothingPearson,smoothingMicrosoft,randompartitionscor}. Unfortunately, most heuristic defenses break in the presence of suitably strong adversaries \citep{bypassing,obfuscatedGradient,breakingDefenses} and provable defenses are often impractical or allow only very small perturbations. Thus a full defense remains elusive. The current literature on this topic is considerable. We refer the reader to \citet{surveyeverything} for an overview of both attacks and defenses and to \citet{blackboxsurvey} for a survey focused on black-box attacks. 

We consider black-box and white-box models as the extreme points of a spectrum parameterized by the number of queries allowed for the adversary. This point of view is related to \citet{limitedqueries} where the authors design a black-box attack with a limited number of queries. Intuitively, the more queries the adversary can make the more knowledge he gains about the classifier. When the number of queries approaches infinity then we transition from a black-box to a white-box model as in this limit the adversary knows the classifying function exactly. Using this point of view we ask:

\begin{center}
\textit{How many queries does the adversary need to make to reliably find adversarial examples?}
\end{center}

By ``reliably'' we mean comparable with the information-theoretic white-box performance. To be more formal, we assume that there is a distribution $\mathcal{D}$ and  a high-accuracy classifier $f$ that maps $\R^d$ to classes $\mathcal{Y}$. The adversary $\adv$ only has black-box access to $f$. Moreover, $\e \in \R_{+}$ is an upper bound on the norm (usually $\ell_p$ norm) of the allowed adversarial perturbation. Assume that $f$ is susceptible to $\e$-bounded adversarial perturbations for an $\eta$-fraction of the underlying distribution $\mathcal{D}$. The quantity $\eta$ is the largest error an adversary, who has access to unbounded computational resources and fully knows $f$, can achieve.
We ask: How many queries to the classifier $f$ does $\adv$ need to make in order to be able to find adversarial examples for say an $\eta/2$-fraction of the distribution $\mathcal{D}$? This question is similar to problems considered in \citet{blackboxcertification}. The difference is that in \citet{blackboxcertification} the authors define the query complexity of the adversary as a function of the number of points for which the adversarial examples are to be found. Moreover, they require the adversary to be perfect, that is to find adversarial examples whenever they exist. This stands in contrast to our approach that only requires the adversary to succeed for say a 1/2 fraction of the adversarial examples. The question we ask is also similar to ideas in \citet{adversarialPAClearning}. In this paper the authors consider a generalization of PAC learning and ask how many queries an algorithm requires in order to learn robustly. Similar questions were also asked in \citet{rademacher} and \cite{robustmoredata}. The difference is that we focus on the query complexity of the attacker and not the defender.

\paragraph{Our contributions.} We introduce a new notion - the \textbf{query complexity} (QC) of adversarial attacks. This notion unifies the two most popular attack models and enables a systematic study of robustness of learning algorithms against query-bounded adversaries. 

Our findings are the following: the higher the entropy of the decision boundaries that are created by the learning algorithm the more secure is the resulting system in our attack model. We first prove a general lower bound on the QC in terms of the entropy of decision boundaries. Then, using this result, we present two scenarios for which we are able to show meaningful lower bounds on the QC. The first one is a simple $2$-dimensional distribution and a nearest neighbor algorithm. For this setting we are able to prove a strong query lower bound of $\Theta(m)$, where $m$ is the number of samples on which the classifier was ``trained''. For the second example we consider the well-known adversarial spheres distribution, introduced in the seminal paper \citet{adversarialSpheres}. For this learning task we argue that quadratic neural networks have a query lower bound of $\Theta(d)$, where $d$ is the dimensionality of the data. We discuss why certain learning algorithms like linear classifiers and also neural networks might be far less secure than KNN against query-bounded adversaries. Finally, we use the lower bound on the QC in terms of entropy to prove, for a broad class of learning algorithms, a security guarantee against query-bounded adversaries that grows with accuracy.  

There exist tasks for which it is easy to find high-accuracy classifiers but finding robust models is infeasible. E.g., in \citet{computationalhardness} the authors describe a situation where it is information-theoretically easy to learn robustly but there is no algorithm in the statistical query model that computes a robust classifier. 
In \citet{computationalhardnessfull} an even stronger result is proven. It is shown that under a standard cryptographic assumption there exist learning tasks for which no algorithm can efficiently learn a robust classifier. Finally, in \citet{robustnessaccuracy} it was shown that robust and accurate classifiers might not even exist. The query-bounded point of view shows a way to address these fundamental difficulties -- even for tasks for which it is impossible to produce a model that is secure against a resource-unbounded adversary, it might be possible to defend against a query-bounded adversary.


\paragraph{Organization of the paper.} In Section~\ref{sec:attackmodel} we formally define the threat model and the query complexity of adversarial attacks. In Section~\ref{sec:easylowerbound} we show that a security guarantee against query-bounded adversaries that grows with accuracy for a rich class of learning algorithms. In Section~\ref{sec:knn} and~\ref{sec:quadraticnetwork} we analyze the query complexity of KNN and Quadratic Neural Network learning algorithms respectively. In Section~\ref{sec:paritions} we present a universal defense against query-bounded adversaries. Finally, in Section~\ref{sec:conclusions} we summarize the results and discuss future directions. We defer most of the proofs to the appendix.

\section{The Query Complexity of adversarial attacks}\label{sec:attackmodel}

We start by formally defining the {\em threat model}. For a {\em data distribution} $\mathcal{D}$ over $\mathbb{R}^d$ and a set $A \subseteq \mathbb{R}^d$ let 
$\mu(A) := \mathbb{P}_{X \sim \mathcal{D}}[X \in A] \text{.}$  For simplicity we consider only {\em separable} binary classification tasks. Such tasks are fully specified by ${\mathcal D}$ as well as a {\em ground truth} $h : \R^d \rightarrow \{-1, 1\}$. 
For a binary classification task with a ground truth $h : 
\mathbb{R}^d \xrightarrow{} \{-1, 1\}$ and a classifier $f : \mathbb{R}^d \xrightarrow{} \{-1, 1\}$ we define the {\em error set} as $E(f) := \{x \in \R^d : f(x) \neq h(x) \} \text{.}$ Note that with this definition it might happen that $E(f) \not\subseteq \text{supp}(\mathcal{D})$. For $x \in \mathbb{R}^d$ and $\epsilon>0$ we write $B_{\e}(x)$ to denote the {\em closed ball} with center $x$ and radius $\e$ and $B_\e$ to denote the {\em closed ball} with center $0$ and radius $\e$. We say that a function $p : \R^d \xrightarrow{} \R^d$ is an {\em $\e$-perturbation} if for all $x \in \R^d$ we have $\| p(x) - x \|_2 \leq \e$. For $n \in \N$ we denote the set $\{1,2,\dots,n\}$ by $[n]$. For $x,y \in \R^d$ we will use $[x,y]$ to denote the closed line segment between $x$ and $y$. For $A,B \subseteq \R^d$ we define $A + B := \{x + y : x \in A, y \in B\}$. We use $m$ to denote the sample size.

\begin{definition}[\textbf{Risk}]
Consider a separable, binary classification task with a ground truth $h : 
\mathbb{R}^d \xrightarrow{} \{-1, 1\}$. For a classifier $f : \mathbb{R}^d \xrightarrow{} \{-1, 1\}$ we define the \textbf{R}isk as  
$R(f) := \mathbb{P}_X [f(X) \neq h(X)] \text{.}$
\end{definition}

\begin{definition}[\textbf{Adversarial risk}] \label{def:adversarialrisk}
Consider a binary classification task with separable classes with a ground truth $h : 
\mathbb{R}^d \xrightarrow{} \{-1, 1\}$. For a classifier $f : \mathbb{R}^d \xrightarrow{} \{-1, 1\}$ and $\epsilon \in \mathbb{R}_{\geq 0}$ we define the \textbf{A}dversarial \textbf{R}isk as:
\[AR(f,\epsilon) := \mathbb{P}_X [ \exists \ \gamma \in B_\epsilon : f(X + \gamma) \neq h(X)] \text{.}\]
an $\e$-perturbation $p$ we define:
\[AR(f,p) :=  \mathbb{P}_X [f(p(X)) \neq h(X)] \text{,}\]
to be the adversarial risk of a specific perturbation function $p$.
\end{definition}


Note: In the literature in the right-hand-side of the adversarial risk definition, the classifier $f(X)$ rather than the ground truth $h(X)$ are used, i.e., $\mathbb{P}_X [ \exists \ \gamma \in B_\epsilon : f(X + \gamma) \neq f(X)]$. Our definition is stricter as it gives an upper bound.
In order to keep the exposition simple, we restrict our discussion to $\ell_2$-bounded adversarial perturbations. Other norms can of course be considered and might be important in practice.

\begin{definition}[\textbf{Query-bounded adversary}] \label{def:queryboundedadversary}
For $\e \in \mathbb{R}_{\geq 0}$ and $f : \R^d \xrightarrow{} \{-1, 1\}$ a $q$-bounded adversary with parameter $\e$ is a deterministic algorithm\footnote{We use {\em algorithm} here since this seems more natural. But we do not limit the attacker computationally nor are we concerned with questions of computability. Hence, {\em function} would be equally correct.} $\adv$ that asks at most $q \in \N$ (potentially adaptive) queries of the form $ f(x) \stackrel{?}{=} 1$ and outputs an $\e$-perturbation $\adv(f) : \R^d \xrightarrow{} \R^d$.
\end{definition}

\begin{definition}[\textbf{Query complexity of adversarial attacks}] \label{def:querycomplexity}
Consider a binary classification task $T$ for separable classes with a ground truth $h : 
\mathbb{R}^d \xrightarrow{} \{-1, 1\}$ and a distribution $\mathcal{D}$. Assume that there is a learning algorithm $\text{ALG}$ for this task that given $S \sim \mathcal{D}^m$ learns a classifier $\text{ALG}(S) : \R^d \xrightarrow{} \{-1, 1\}$. For $\e \in \mathbb{R}_{\geq 0}$ define the Query Complexity of adversarial attacks on $\text{ALG}$ with respect to $(T,m,\e)$ and denote it by $QC(\text{ALG},T,m,\e)$: It is the minimum $q \in \N$ so that there exists a $q$-bounded adversary $\adv$ with parameter $\e$ such that $\mathbb{P}_{S \sim \mathcal{D}^m}$ of the event
\[AR(\text{ALG}(S),\adv(\text{ALG}(S))) \geq \frac{1}{2} AR(\text{ALG}(S),\e) \]
is at least $0.99 \text{.}$
\end{definition}
In words, it is the minimum number of queries that is needed so that there exists an adversary who can achieve an error of half the maximum achievable error (with high probability over the data samples). Note that it follows from Definitions~\ref{def:queryboundedadversary} and \ref{def:querycomplexity} that $\adv$ is computationally unbounded, knows the distribution $\mathcal{D}$ and the ground truth $h$ of the learning task and also knows the learning algorithm $\text{ALG}$. The only restriction that is imposed on $\adv$ is the number of allowed queries. What is important is that $\adv$ does {\em not} know $S$ nor the potential randomness of $\text{ALG}$ (in the generalized setting $\text{ALG}$ can be randomized, see Definition~\ref{def:querycomplexity2}) -- this is what makes the QC non-degenerate. To see this, observe that if $\adv$ knew $S$ and $\text{ALG}$ was deterministic then $\adv$ could achieve $AR(\text{ALG}(S),\adv(\text{ALG}(S))) = AR(\text{ALG}(S),\e)$ without asking any queries. This is because $\adv$ can for every point $x$ check if there exists $\gamma \in B_\epsilon$ such that $\text{ALG}(S)(X + \gamma) \neq h(X)$, as $\adv$ can compute $\text{ALG}(S)$ without asking any queries. This allows $\adv$ to achieve adversarial risk of $AR(\text{ALG}(S),\e)$ (see Definition~\ref{def:queryboundedadversary}).

Defnition~\ref{def:querycomplexity} was guided by experiments. For instance, in \citet{goodfellowblackbox}, in order to attack a neural network $f$ the adversary trains a new neural network $\hat{f}$. This $\hat{f}$ acts as an approximation of $f$. She does so by creating a training set, which is labeled using $f$ and then training $\hat{f}$ on this dataset. The concept of transferability allows the adversary to ensure that $f$ will also misclassify inputs misclassified by $\hat{f}$. This phenomenon is reflected in our definition -- after the initial phase of querying $f$ (and training $\hat{f}$) the adversary no longer has access to $f$. This means that after this phase the adversary implicitly constructed an $\epsilon$-perturbation $p$ as in Definition 2. One can also consider different definitions. For instance ask about the number of queries required to attack a {\em given point} or measure the adversarial risk in {\em absolute terms} (instead of 1/2 of the white-box performance). Both of these questions are valid and are of interest in their own right. For the sake of definiteness we choose what we consider to be one important viewpoint, a viewpoint that is well motivated by experiments.  

Definition~\ref{def:querycomplexity} can be generalized to incorporate randomness in the learning algorithm. 
Intuitively, the randomness in $\text{ALG}$ can increase the entropy of the learning process and that in turn may lead to a higher QC. Further, both the approximation constant (which is chosen to be $1/2$ in Definition~\ref{def:querycomplexity}) as well as the success probability can also be generalized.
%
%
\begin{definition}[\textbf{Query complexity of adversarial attacks - generalized}] \label{def:querycomplexity2}
Consider a binary classification task $T$ for separable classes with a ground truth $h : 
\mathbb{R}^d \xrightarrow{} \{-1, 1\}$ and a distribution $\mathcal{D}$. Assume that there is a \textbf{randomized} learning algorithm $\text{ALG}$ for this task that given $S \sim \mathcal{D}^m$ and a sequence of random bits $B \sim \mathcal{B}$ learns a classifier $\text{ALG}(S,B) : \R^d \xrightarrow{} \{-1, 1\}$. For $\e \in \mathbb{R}_{\geq 0}, \kappa, \alpha \in [0,1]$ define the Query Complexity of the adversarial attacks on $\text{ALG}$ with respect to $(T,m,\e,\alpha,\kappa)$ and denote it by $QC(\text{ALG},T,m,\e,\alpha,\kappa)$: It is the minimum $q \in \N$ such that there exists a $q$-bounded adversary $\adv$ with parameter $\e$ such that $\mathbb{P}_{S \sim \mathcal{D}^m, B \sim \mathcal{B}}$ of the event

\[AR(\text{ALG}(S,B),\adv(\text{ALG}(S,B))) \geq \alpha AR(\text{ALG}(S,B),\e) \]
is at least $1 - \kappa$.

If the above holds for $\adv$ we will refer to $\alpha$ as the \textbf{approximation constant} of $\adv$ and to $\kappa$ as the \textbf{error probability} of $\adv$.
\end{definition}

For the sake of clarity whenever possible we will restrict ourselves for the most part to Definition~\ref{def:querycomplexity}. This eliminates two parameters from our expressions and restricts the attention to deterministic learning algorithms. Only when the distinction becomes important will we refer to Definition~\ref{def:querycomplexity2}.

\paragraph{Summary:} The query complexity of adversarial attacks is the minimum $q$ for which there exists a $q$-bounded adversary that carries out a successful attack. Such adversaries are computationally unbounded, know the learning task and the learning algorithm but \textit{don't} know the training set. 

\section{High-entropy decision boundaries lead to robustness}\label{sec:entropy}

The decision boundary of a learning algorithm applied to a given task can be viewed as the outcome of a random process: (i) generate a training set and, (ii) apply to it the, potentially randomized, learning algorithm. Recall, see Definitions~\ref{def:querycomplexity} and \ref{def:querycomplexity2}, that a query-bounded adversary does not know the sample on which the model was trained nor the randomness used by the learner. This means that if the decision boundary has high entropy then the adversary needs to ask many questions to recover the boundary to a high degree of precision. This suggest that high-entropy decision boundaries are robust against query-bounded adversaries since intuitively it is clear that an approximate knowledge of the decision boundary is a prerequisite for a successful attack.

Next we present a result that makes this intuition formal. Before delving into the details let us explain the intuition behind this approach. Let us recall the set-up. The classifier is trained on a sample $S$ that is unknown to the adversary. This classifier has a particular error set. We say that an adversary succeeds if, after asking some queries, it manages to produce an $\epsilon$-perturbation with the property that this perturbation moves ``sufficient'' mass into the error set of the classifier.  Here, sufficient means at least half of what is possible if the adversary had known the classifier exactly. Let us say in this case that an $\epsilon$-perturbation is {\em consistent} with an error set. 

The following theorem states that if for every $\e$-perturbation the probability that an error set of a classifier is consistent with that perturbation is small then the QC is high. This is true since if for every $\e$-perturbation only a small fraction of probability space (i.e., the possible classifiers) is consistent with this perturbation then $\adv$'s protocol has to return many distinct $\e$-perturbations depending on the outcome of its queries. And to distinguish which perturbation it should return it has to ask many queries. 

\begin{restatable}{theorem}{thmreduction}[\textbf{Reduction.}]
\label{thm:reduction}
Let $\e \in \mathbb{R}_{\geq 0}$ and let $T$ be a binary classification task on $\R^d$ with separable classes. Let $ALG$ be a randomized learning algorithm for $T$ that uses $m$ samples. Then for every $\kappa \in [0,1]$ the following holds:
\begin{align*}
&QC(ALG,T,m,\e,1/2,\kappa) \\
&\geq \log \left(\frac{1 - \kappa}{\sup_{p : \text{ $\e$-perturbation}} \mathbb{P}_{S \sim \mathcal{D}^m, B \sim \mathcal{B}} \left[ \mathcal{E}(S,B,p) \right]} \right) \text{,}
\end{align*}
where the event $\mathcal{E}(S,B,p)$ is defined as:
$$ \mu(p^{-1}(E(ALG(S,B)))) \geq \frac{AR(ALG(S,B),\e)}{2} \text{.}$$
\end{restatable}

\begin{remark}
For the sake of clarity and consistency with the standard setup we fixed the approximation constant to be equal $1/2$. We note however, that Theorem~\ref{thm:reduction} (and its proof) is also true for all approximation constants.
\end{remark}

\paragraph{Summary:} Theorem~\ref{thm:reduction} is a key ingredient in most of our results. It gives a lower bound on the QC in terms of a geometric-like notion of entropy of error sets. This is often much easier to compute than to analyze the inner workings of a particular learning algorithm.

\section{Entropy due to locality -- K-NN algorithms}\label{sec:knn}

Let us now analyze the QC of $K$-Nearest Neighbor (K-NN) algorithms. Nearest neighbor algorithms are among the simplest and most studied algorithms in machine learning. They are also widely used as a benchmark.  
It was shown in \citet{knnbound} that for a sufficiently large training set, the risk of the $1$-NN learning rule is upper bounded by twice the optimal risk. It is also known that these methods suffer from the "curse of dimensionality" -- for $d$ dimensional distributions they typically require $m = 2^{\Theta(d \log(d))}$ many samples. That is why in practice one often uses some dimensionality reduction subroutine before applying $K$-NN. Moreover, K-NN techniques are one of the few learning algorithms that do not require any learning. In the naive implementation all computation is deferred until function evaluation. This is related to the most interesting fact from our perspective, namely that \textbf{the classification rule of the K-NN algorithm depends only on the local structure of the training set}. 

We argue that this property makes K-NN algorithms secure against query-bounded adversaries. Intuitively, if the adversary $\adv$ wants to achieve a high adversarial risk she needs to understand the global structure of the decision boundary. But if the classification rule is only very weakly correlated between distant regions of the space then this intuition suggests that $\adv$ may need to ask $\Theta(m)$ many queries to guarantee a high adversarial risk. This is consistent with the entropy point of view.
Moreover there are experimental results (see
\citet{knnattack,papernottransferability}) that show that it is hard to attack K-NN classifiers in the black-box model.

We make these intuitions formal in the following sense. We design a synthetic binary learning task in $\R^2$, where the data is uniformly distributed on two parallel intervals -- correspondiing to the two classes. We then show a $\Theta(m)$ lower bound for the QC of $1$-NN algorithm for this learning task. This means that the number of queries the adversary needs to make to attack $1$-NN is proportional to the number of samples on which the algorithm was "trained". We conjecture that a similar behavior occurs in higher dimensions as well. 

\subsection{K-NN -- QC lower bounds}

Consider the following distribution. Let $m \in \N$ and $z \in \R_{+}$. Let $L_-,L_+ \subseteq \R^2$ be two parallel intervals of length $m$ placed at distance $z$ apart. More formally, $L_- := [(0,0), (m,0)]$, 
$L_+ = [(0,z), (m,z)]$.
Let the binary learning task $T_{\text{intervals}}(z)$ be as follows. We generate $\bar{x} \in \R^2$ uniformly at random from the union $L_- \cup L_+$. We assign the label $y = -1$ if $\bar{x} \in L_-$ and $y = +1$ otherwise. In Figure~\ref{fig:decisionboundaryknn} we visualize a decision boundary of the $1$-NN algorithm for a random $S \sim \mathcal{D}^m$ on the $T_{\text{intervals}}$. Horizontal lines, black and gray, represent the two classes, crosses are data points, white and gray regions depict the classification rule and the union of red intervals is equal to the error set. We also include more visualizations in the appendix. The main result of this subsection is:

\begin{restatable}{theorem}{thmknn}
\label{thm:knn}
There exists a function $\lambda : \R^+ \xrightarrow[]{} (0,1)$ such that the $1$-Nearest Neighbor ($1$-NN) algorithm applied to the learning task $T_{\text{intervals}}(z)$ satisfies:
\[
QC(1\text{-NN},T_{\text{intervals}}(z),2m,z/10,1-\lambda,0.1) \geq \Theta(m) \text{,}\]
provided that $z = \Omega(1)$.
\end{restatable}

\begin{figure*}
  \centering
  \includegraphics[width=1\textwidth]{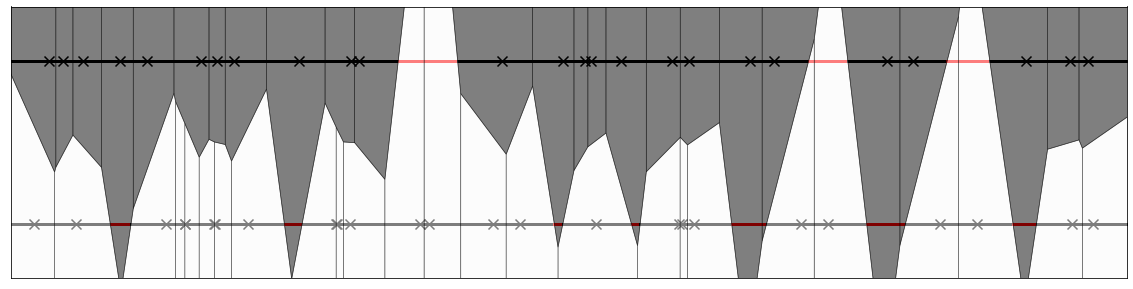}
  \caption{A random decision boundary of $1$-NN for $T_{\text{intervals}}$.}
  \label{fig:decisionboundaryknn}
\end{figure*}

\paragraph{Summary:} The K-NN algorithm learns classification rules that depend only on the local structure of the data. This implies high-entropy decision boundaries, which in turn leads to robustness against query-bounded adversaries. The QC of 1-NN scales at least linearly with the size of the training set.

\section{Entropy due to symmetry - Quadratic Neural Networks}\label{sec:quadraticnetwork}


In this section we analyze the QC of Quadratic Neural Networks (QNN) applied to a learning task defined in \citet{adversarialSpheres}. Let $S_r^{d-1} := \{ x \in \R^d : \|x\|_2 = r\}$. The distribution $\mathcal{D}$ is defined by the following process: generate $x \sim U[S_1^{d-1}]$ and $b \sim U\{-1,1\}$ (where $U$ denotes the uniform distribution). If $b = -1$ return $(x,-1)$, otherwise return $(1.3 x,+1)$. The associated ground truth is defined as $h(x) = -1$ for $x \in \R^d, \|x\|_2 \leq 1.15$ and $h(x) = 1$ otherwise .

The QNN is a single hidden-layer network where the activation function is the quadratic function $\sigma(x) = x^2$. There are no bias terms in the hidden layer. The output node computes the sum of the activations from the hidden layer, multiplies them by a scalar and adds a bias. If we assume that the hidden layer contains $h$ nodes then the network has $d \cdot h + 2$ parameters. It was shown in \citet{adversarialSpheres} that the function that is learned by QNN has the form
$y(x) = \sum_{i=1}^d \alpha_i z_i^2 - 1 \text{,}$
where the $\alpha_i$'s are scalars that depend on the parameters of the network and $z = M(x)$ for some orthogonal matrix $M$. The decision boundary is thus $\sum_{i=1}^d \alpha_i z_i^2 = 1$, which means that it is an ellipsoid centered at the origin.

In a series of experiments performed for the Concentric Spheres (CS) dataset in \citet{adversarialSpheres} it was shown that a QNN trained with $N = 10^6$ many samples with $d = 500$ and $h = 1000$ learns a classifier with an estimated error of approximately $10^{-20}$ but the adversarial risk $\eta$ is high and is estimated to be $1/2$ when $\e \approx 0.18$. On the theoretical side, it was proven in \citet{randompartitionscor} (see Section 9.1) that
\begin{equation}\label{eq:bestepsilon}
\e \leq O\left(\frac{\log(\eta/\delta)}{d} \right) \text{.}
\end{equation}
In words, \eqref{eq:bestepsilon} gives an upper bound on the biggest allowed perturbation $\e$ in terms of the risk $\delta$ the adversarial risk $\eta$ and the dimension $d$. In particular if we want the classifier to be adversarially robust for $\e = \Theta(1)$ (that is for perturbations comparable with the separation between the two classes) then $\delta = 2^{-\Omega(d)}$. Even robustness of only $\e = \Theta(1/\sqrt{d})$ requires the risk to be as small as $\delta = 2^{-\Omega(\sqrt{d})}$. These results paint a bleak picture of the adversarial robustness for CS.  

The QC point-of-view is more optimistic. 
Using results from Section~\ref{sec:easylowerbound} we first show that if the network learns classifiers with risk $2^{-\Omega(k)}$ then it automatically leads to a lower bound on the QC of $\Theta(k)$. Moreover, for a simplified model of the network, we show that even if the risk of the learned classifier is only a small constant, say $0.01$, then this results in a lower bound on the QC of $\Theta(d)$ for perturbations of $\Theta(1/\sqrt{d})$. Using \eqref{eq:bestepsilon} our result guarantees security against $\Theta(d)$-bounded adversaries for perturbations which are $\Theta(\sqrt{d})$ times bigger than the best possible against unbounded adversaries. This shows that restricting the power of the adversary can make a significant difference.

We argue that the obtained $\Theta(d)$ lower bound is close to the real QC for this algorithm and learning task. Observe that the decision boundary of the network is an ellipsoid which can be described by $O(d^2)$ parameters ($d^2$ for the rotation matrix and $d$ for lengths of principal axes). This suggest that it should be possible to design a $O(d^2)$-bounded adversary that succeeds on this task. Assuming that this is indeed the case, our lower bound is only a factor $O(d)$ away from the optimum. 

The results of this section can be understood in the following way. 
The relatively simple structure of the decision boundaries allows the adversary to attack the model with only $O(d^2)$ queries. There is however enough entropy in the network to guarantee a lower bound for the QC of $\Theta(d)$. This entropy intuitively comes from the rotational invariance of the dataset and in turn of the learned decision boundary. We conjecture that algorithms like linear classifiers (e.g., SVMs) exhibit a similar behavior. That is, for natural learning tasks they are robust against $q$-bounded adversaries only for $q = O(\text{poly}(d))$. The reason is that all these algorithms generate classifiers with relatively simple decision boundaries which can be described by $O(\text{poly}(d))$ parameters.

But this is not the end of the story for CS. Our results don't preclude the possibility that there exist a learning algorithm that is secure against $q$-bounded adversaries for $q \gg d$. 
In fact in Section~\ref{sec:paritions} we present an off-the-shelf solution that can be applied to CS dataset and which, by injecting entropy, achieves security against $k$-bounded adversaries for $\e = \Theta \left(\frac{1}{k \sqrt{d}} \right)$.

\subsection{Quadratic Neural Networks -- QC lower bounds}

Using the results from Section~\ref{sec:easylowerbound} one can show that increased accuracy leads to increased robustness. More precisely if QNN has a risk of $2^{-\Omega(k)}$ then it is secure against $\Theta(k)$-bounded adversaries for $\e = \Theta(1)$. The proof of this fact is deferred to the appendix.

Now we argue that also in the case where the risk achieved by the network is as large as a constant then QNN are still robust against $\Theta(d)$-bounded adversaries. We first argue that any reasonable optimization algorithm applied to QNN for the CS learning task gives rise to a distribution on error sets that is rotational invariant. This follows from the fact that $\mathcal{D}$ itself is rotational invariant. Now observe that for QNN the error sets are of the form: 
$\left\{x \in \R^d : \|x\|_2 \leq 1.15,  \sum_{i=1}^d \alpha_i z_i^2 > 1 \right\} \cup \left\{x \in \R^d : \|x\|_2 > 1.15, \sum_{i=1}^d \alpha_i z_i^2 < 1 \right\} \text{,}$ as the decision boundary learned by QNN is defined by $\sum_{i=1}^d \alpha_i z_i^2 = 1$, where $z = Mx$ for some orthonormal matrix $M$. These sets might be quite complicated as they are basically defined as the set difference of a ball and an ellipsoid. We will refer to the real distribution on error sets of QNN as $\mathcal{E}_{\text{QNN}}$. We assume that the standard risk of classifiers learned by the QNN is concentrated around a constant $\delta$.

Intuitively a ``complicated'' (high entropy) distribution on error sets results in a high QC and a ``simple'' (low entropy) distribution results in a low QC. In the rest of this section we first introduce a set of ``simple'', artificial distributions over error sets and then we state QC lower bounds for these distributions. Formal definitions are presented in Definition {\em Distributions on Spherical Caps} in the appendix, here we give a short description of what they are. For $y \in S_1^{d-1}$ let $\text{cap}(y,r,\tau) := B_r \cap \{x \in \R^d : \langle x,y \rangle \geq \tau \}$. Let $\tau : [0,1] \xrightarrow{} [0,1]$ be such that for every $\delta \in [0,1]$ we have $\nu(\text{cap}(\cdot,1,\tau(\delta))) / \nu(S_1^{d-1}) =  \delta$, where $\nu$ is a $d-1$ dimensional measure on the sphere $S_1^{d-1}$. For $k \in \N_+$ let:
\[E_-(k) := \text{cap}(e_1, 1.15, \tau(\delta/k)) \setminus B_{1.15/1.3} \]
\[E_+(k) := \text{cap}(e_1, 1.495, 1.3\tau(\delta/k)) \setminus B_{1.15} \text{,}\]
where $e_1$ is a standard basis vector. Note that for every $k$ we have $1.3 \cdot E_-(k) = E_+(k)$.
For $\delta \in (0,1), k \in \N_+$ the distributions are: $\text{Cap}(\delta)$ - randomly rotated either $E_-(1)$ or $E_+(1)$, chosen uniformly at random; $\text{Cap}_k^{i.i.d}(\delta)$ - union of $k$ i.i.d. randomly rotated sets each either $E_-(k)$ or $E_+(k)$, chosen uniformly at random; $\text{Cap}_k^{\mathcal{G}}(\delta)$ - $k$ randomly rotated sets, each either $E_-(k)$ or $E_+(k)$, chosen uniformly at random; the relative positions of cap's normal vectors are determined by $\mathcal{G}$, where $\mathcal{G}$ is a given distribution on $(S_1^{d-1})^k$.

We conjecture that $\text{Cap}(0.01)$, $\text{Caps}_d^{\text{i.i.d.}}(0.01)$, $\text{Caps}_d^{\mathcal{G}}(0.01)$ have QCs that are no larger than the QC of $\mathcal{E}_{\text{QNN}}$ that achieves standard risk $0.01$. The intuitive reason is that they contain less entropy than $\mathcal{E}_{\text{QNN}}$ and so it should be easier to attack these distributions. In Lemma~\ref{lem:lwrbndforcap} we prove a $\Theta(d)$ lower bound for Cap$(0.01)$ and, in the appendix, we give a matching upper bound of $\Theta(d)$. Also in the appendix, we give two reductions that lower-bound QC of $\text{Caps}_d^{\text{i.i.d.}}$ and $\text{Caps}_d^{\mathcal{G}}$ based on a conjecture (see {\em $\text{Cap}$ conjecture} in the appendix). We summarize the proved and conjectured lower bounds in Table~\ref{tab:sphereslwrbnds}.

\begin{restatable}[\textbf{Lower bound for $\text{Cap}$}]{lemma}{lwrbndforcap}
\label{lem:lwrbndforcap}
There exists $\lambda > 0$ such that if a $q$-bounded adversary $\adv$ succeeds on $\text{Cap}(0.01)$ with approximation constant  $ \geq 1 - \lambda$, error probability $2/3$ for $\e = \tau(0.01)$. Then
\[q \geq \Theta(d) \text{.}\]
\end{restatable}



\begin{table}
  \caption{QC for CS}
  \label{tab:sphereslwrbnds}
  \centering
  \begin{tabular}{ll}
    \\
    \toprule
    Error distribution     & Lower bound     \\[0.5mm]
    \midrule
    $\text{Cap}(0.01)$             & $\Theta(d)$        \\[0.5mm] 
    $\text{Caps}_k^{\text{i.i.d.}}(0.01)$             & $\Theta(d)$\tablefootnote{This lower bound is conditional on {\em $\text{Cap}$ conjecture} (in the appendix).\label{reference footnote}}  \\[0.5mm]
    $\text{Caps}_k^{\mathcal{G}}(0.01)$         & $\Theta(d/k)$\textsuperscript{†} \\[0.5mm]
    \bottomrule
  \end{tabular}
\end{table}

\paragraph{Summary:} Quadratic neural networks have simple decision boundaries - they are of the form of ellipsoids. But due to the rotational symmetry there is sufficient entropy to guarantee robustness against $\Theta(d)$-bounded adversaries.

\section{How to increase the entropy of an existing scheme -- a universal defense }\label{sec:paritions}

It was proven in \citet{randompartitionscor} that there exists a universal defense against adversarial attacks. The defense algorithm gets as an input access to a high accuracy classifier $f$ and outputs a new classifier $g$ that  is adversarially robust. 
The idea of the defense is based on randomized smoothing \citep{smoothingPearson,smoothingMicrosoft} and random partitions of metric spaces. Simple rephrasing of Theorem 5 from \citet{randompartitionscor} in the language of the QC of adversarial attacks gives the following:

\begin{theorem}
For every $d \in \N_+$ there exists a randomized algorithm $\text{DEF}$ satisfying the following. It is given as input access to an initial classifier $\R^d \xrightarrow{} \{-1,1\}$ and provides oracle access to a new classifier $\R^d \xrightarrow{} \{-1,1\}$. For every separable binary classification task $T$ in $\R^d$ with separation $\e$ the following conditions hold. Let $\text{ALG}$ be a learning algorithm for $T$ that uses $m$ samples. Then for every $S \sim \mathcal{D}^m$ we have $R(\text{DEF}(\text{ALG}(S))) \leq 2R(\text{ALG}(S))$ and for every $\e' > 0$:
\[QC(\text{DEF} \circ \text{ALG},T,m,\e') \geq \Theta \left(\frac{\e}{\sqrt{d} \cdot \e'} \right) \text{.}\]

\end{theorem}



\paragraph{Summary:} There exists a universal defense that can be applied on top of any learning algorithm to make it secure against query-bounded adversaries. Roughly speaking, it works by injecting additional randomness to increase the entropy of the final classifier. 

\section{Robustness and accuracy -- foes no more}\label{sec:easylowerbound}

It was argued in \citet{robustnessaccuracy} that there might be an inherent tension between accuracy and adversarial robustness. We argue that this potential tension disappears for a rich class of learning algorithms if we consider $q$-bounded adversaries. We show that if a learning algorithm satisfies a particular natural property then there is a lower bound for the QC of this algorithm that {\em grows} with accuracy.

\begin{restatable}{theorem}{thmlogoverdeltabound}
\label{thm:log1overdeltabound}
For every $\e \in \mathbb{R}_{\geq 0}, C,\delta,\eta \in \R_{+}$ and $T$ a binary classification task on $\R^d$ with separable classes the following conditions hold. If $\text{ALG}$ is a learning algorithm for $T$ and satisfies the following properties:
\begin{enumerate}
    \item 
    $\forall x \in \text{supp}(\mathcal{D}) + B_\e, \\ \mathbb{P}_{S \sim \mathcal{D}^m}[\text{ALG}(S)(x) \neq h(x)] \leq C \cdot \delta \text{,}$ 
    \label{item:firstprop}
    \item $\mathbb{P}_{S \sim \mathcal{D}^m}[AR(\text{ALG}(S),\e) \geq \eta] \geq 0.99 \text{,}$ \label{item:secondprop}
    \item $\mathbb{P}_{S \sim \mathcal{D}^m} [R(\text{ALG}(S)) \leq \delta] \geq 0.99 \text{,}$ \label{item:thirdprop}
\end{enumerate}
then: 
\[QC(\text{ALG},T,m,\e) \geq \log \left( \frac{ \eta}{3 \cdot C \cdot \delta} \right) \text{.}\]
\end{restatable}
The lower bound obtained in Theorem~\ref{thm:log1overdeltabound} is useful in situations when $\text{ALG}(S)$ has high accuracy but the adversarial risk is large. This is a typical situation when using neural networks -- one is often able to find classifiers that have high accuracy but they are not adversarially robust. 

\paragraph{Summary:} For a rich class of learning algorithms our security guarantee against query-bounded adversaries increases with accuracy. A risk of $2^{-\Omega(k)}$ leads to robustness against $\Theta(k)$-bounded adversaries.

\section{Discussion}


For a given task the QC can vary considerably depending on the learning algorithm. In this section we try to explain our current understanding of what governs this dependence.

Consider the two intervals learning task from Theorem 2. As proven, if we use the 1-NN classifier the QC is lower-bounded by $\Theta(m)$, where $m$ is the number of 'training' examples. Now consider what happens if we used the hypothesis class of linear separators with the standard ERM algorithm. Then we expect the learned classifier to be a line that approximately separates the two intervals. To approximately recover this line the adversary can find two points through which the line passes by running two binary search procedures. This implies that the QC is independent of $m$. Thus QCs can be as different as $\Omega(m)$ and $O(1)$ depending on the used hypothesis class/learning algorithm. Note that if we used the SVM classifier then with allowed perturbation of $z/10$ (as in Theorem 2) the adversarial risk will be $0$ (for $m$ big enough), as the learned classifier will be defined by the line passing through $(0,z/2)$ and $(m,z/2)$, which makes the QC degenerate and equal $0$.

Consider now the CS data set and see what happens when you very the number of samples $m$ and the learning algorithm. As we mentioned in Section~\ref{sec:knn}, for CS K-NN might need as many as $m = 2^{\Theta(d \log d)}$ samples. This means that the lower bound of $\Theta(m)$ (from Theorem~\ref{thm:knn}) becomes $2^{\Theta(d \log d)}$. On the other hand, QNN applied to this task is conjectured to have QC of $\Theta(d^2)$. We can also look at an interesting learning algorithm from \citet{srebroimproperly}, which for a hypothesis class $\mathcal{H}$ can improperly learn $\mathcal{H}$ to a constant adversarial risk using $O(\text{VC-dim}(\mathcal{H}) \cdot  \text{dualVC-dim}(\mathcal{H}))$ many samples. As the VC-dimension, and also the dual VC-dimension, of ellipsoids is in $O(d^2)$ we get that this algorithm achieves constant adversarial risk using $m = O(d^4)$ samples. The QC in this case is, in a way, irrelevant because the adversarial risk is small already. Thus we see three algorithms with different QCs but it's not clear if we can directly compare them as they require different sample sizes.

These examples suggest that there is a strong connection between the number of degrees of freedom of the classifier and the QC. However this connection cannot be expressed in terms of VC-dim or AdversarialVC-dim \citep{adversarialPAClearning} as there exist hypothesis classes and learning tasks for which QC $\gg$ VC-dim/AdversarialVC-dim. Understanding the connection between the QC and the number of parameters is a part of ongoing work.

\section{Conclusions and takeaways}\label{sec:conclusions}

We investigate robustness of learning algorithms against query-bounded adversaries. We start by introducing a definition of QC of adversarial attacks and then proceed to study it's properties. We prove a general lower bound for the QC in terms of an entropy-like property. Using this result we then show a series of lower bounds for classical learning algorithms. Specifically,  we give a lower bound of $\Theta(d)$ for QNNs and  a lower bound of $\Theta(m)$ for $1$-NN algorithm. Moreover we observe that sources of the entropy can be varied. For K-NN the entropy is high due to the locality of the learning algorithm whereas for QNN it comes from the rotational symmetry of the data. The entropy can also be increased by introducing randomness in the learning algorithm itself. We also show that improvements in accuracy of a model lead to an improved security against query-bounded adversaries.

Our analysis identifies properties of learning algorithms that make them (non-)robust. These results give a rule-of-thumb: "The higher the entropy of decision boundary the better" for assessing the QC of a given algorithm.

We believe that a systematic investigation of learning algorithms from the point of view of QC will lead to more adversarially-robust systems. Specifically, it should be possible to design generic defenses that can be applied on top of any learning algorithm. One example of such a defense was given in Section~\ref{sec:paritions}. Significantly more work is needed in order to fulfill the potential of this approach. But imagine that this type of defense could be applied efficiently with only a black-box access to the underlying classifier. And imagine further, that it could guaranteed a QC of, say $q = 2^{\Theta(d)}$. This would arguably solve the adversarial robustness problem.

\bibliography{example_paper}
\bibliographystyle{icml2021}



\onecolumn
\begin{appendix}


\section{Omitted proofs - Reduction and a lower bound}
\label{sec:proofs}

We now present the reduction that we will use for our QC lower bounds later on.



\thmreduction*


\begin{proof}
We first prove the Theorem when $ALG$ is deterministic. Let $\adv$ be a $q$-bounded adversary that performs a successful attack on $ALG$ with respect to $(T,m,\e,1/2, \kappa)$ (as per Definition~\ref{def:querycomplexity2}). We will show that $q$ is lower-bounded by the value from the statement of the Theorem. 

The behavior of $\adv$ can be represented as a binary tree $\mathcal{T}$ where each non-leaf vertex $v \in \mathcal{T}$ contains a query point $x_v \in \R^d$ and each leaf $l \in \mathcal{T}$ contains an $\e$-perturbation $p_l : \mathbb{R}^d \xrightarrow{} \R^d$. Then $\adv$ works as follows: it starts in the root $r$ of $\mathcal{T}$ and queries the vertex $x_r$. Depending on $ f(x_r) \stackrel{?}{=} 1 $ it proceeds left or right. It continues in this manner, querying the points stored in the visited vertices until it reaches a leaf $l$. At the leaf it outputs the perturbation function $p_l$.

Let us partition all possible data sets $S \in (\R^d)^m$ depending on which leaf is reached by $\adv$ when interacting with $ALG(S)$. Let $l_1, \dots, l_n$ be the leaves of $\mathcal{T}$ and $C_1, \dots, C_n \subseteq (\R^d)^m$ be the respective families of data sets that end up in the corresponding leaves. Let $Z := \{S \in (\R^d)^m : \adv \text{ succeeds on } S \}$. By assumption  $\adv$ is guaranteed to succeed with probability $1 - \kappa$, so 
\begin{equation}\label{eq:succesprob}
\mathbb{P}_{S \sim \mathcal{D}^m}[S \in Z] \geq 1 - \kappa \text{.}
\end{equation} 

Now observe that for every $i \in [n]$ and $S \in C_i \cap Z$
\[\mu(p_{l_i}^{-1}(E(ALG(S)))) \geq \frac{AR(ALG(S),\e)}{2} \text{.}\]
In words, for every $S \in C_i \cap Z$ the adversary $\adv$ succeeds if at least $\frac{AR(ALG(S),\e)}{2}$ of the probability mass of $\mathcal{D}$ is moved by $p_{l_i}$ into the error set of $ALG(S)$. Thus we get that for every $i \in [n]$:
\begin{equation}\label{eq:probupperbnd}
 \mathbb{P}_{S \sim \mathcal{D}^m}[S \in C_i \cap Z] \leq \sup_{p : \text{ $\e$-perturbation}} \mathbb{P}_{S \sim \mathcal{D}^m} \left[ \mu(p^{-1}(E(ALG(S)))) \geq \frac{AR(ALG(S),\e)}{2} \right] \text{.}
\end{equation}

By standard properties of entropy we know that for a discrete random variable $W$ any protocol asking yes-no questions that finds the value of $W$ must on average ask at least $H(W)$ many questions. Let $W$ be a random variable that takes values in $\{1,2,\dots,n\}$, where 
for every $i \in [n]$ we have $\mathbb{P}[W = i] := \mathbb{P}_{S \sim \mathcal{D}^m}[S \in C_i \cap Z] / \mathbb{P}_{S \sim \mathcal{D}^m}[S \in Z]$. Note that $\adv$'s protocol can be directly used to find a protocol that asks yes-no questions and finds the value of $W$ with at most $q$ queries. It is enough to prove a lower-bound on $H(W)$ as the expected number of questions can only be lower than the maximum number.

Note that by \eqref{eq:succesprob} and \eqref{eq:probupperbnd} we get that for every $i \in [n]$: 
\[\mathbb{P}[W = i] \leq \frac{1}{1 - \kappa} \cdot \sup_{p : \text{ $\e$-perturbation}} \mathbb{P}_{S \sim \mathcal{D}^m} \left[ \mu(p^{-1}(E(ALG(S)))) \geq \frac{AR(ALG(S),\e)}{2} \right] \text{.}\]
By properties of entropy we know that $H(W) \geq \log(1/\max_{i \in [n]} \mathbb{P}[W = i])$, so in the end we get that:
\[H(W) \geq \log \left(\frac{1 - \kappa}{\sup_{p : \text{ $\e$-perturbation}} \mathbb{P}_{S \sim \mathcal{D}^m} \left[ \mu(p^{-1}(E(ALG(S)))) \geq \frac{AR(ALG(S),\e)}{2} \right]} \right) \text{.}\]

The proof for the case when $ALG$ is randomized in analogous. The only difference is that instead of partitioning the space $(\R^d)^m$ we partition the space $(\R^d)^m \times \text{supp}(\mathcal{B})$.
\end{proof}

\begin{remark}\label{rem:aboutreduction}
For the sake of clarity and consistency with the standard setup we fixed the approximation constant to be equal $1/2$ and the data generation process to be $S \sim \mathcal{D}^m$. We note however, that Theorem~\ref{thm:reduction} (and its proof with minor changes) is also true for all approximation constants and for general data generation processes. By different generation process we mean anything different from $S \sim \mathcal{D}^m$, for instance a case where samples are dependent or where the number of samples $m$ is itself a random variable. This distinction will become important in the proof of Theorem~\ref{thm:knn}.
\end{remark}

The following theorem states that if an algorithm $\text{ALG}$ applied to a learning task satisfies the following: ALG learns low-risk classifier with constant probability, the adversarial risk is high with constant probability and every point from the support of the distribution is misclassified with small probability then the QC of $\text{ALG}$ is high. The core of the proof is the reduction from Theorem~\ref{thm:reduction}.

\thmlogoverdeltabound*

\begin{proof}
Let $p : \mathbb{R}^d \xrightarrow{} \R^d$ be an $\e$-perturbation. For simplicity we introduce the notation $\rho := \mathbb{P}_{S \sim \mathcal{D}^m}[AR(ALG(S),\e) \geq \eta \wedge R(ALG(S),\e) \leq \delta]$. We define two new data distributions: 
\[\mathcal{D}_1 := \mathcal{D}^m | \left( AR(ALG(S),\e) \geq \eta \wedge R(ALG(S),\e) \leq \delta \right)\text{,} \] \[ \mathcal{D}_2 := \mathcal{D}^m | \left( AR(ALG(S),\e) < \eta \vee R(ALG(S),\e) > \delta \right) \text{.}\] 
Observe that $\text{supp}(\mathcal{D}_1) \cap \text{supp}(\mathcal{D}_2) = \emptyset$ and:
\begin{equation}\label{eq:splittingdistribution}
\mathcal{D}^m = \rho \cdot \mathcal{D}_1 + (1- \rho) \cdot \mathcal{D}_2 \text{.}
\end{equation}
Let $\adv$ be an adversary that succeeds on $\mathcal{D}^m$ with probability $0.99$. By  \eqref{eq:splittingdistribution} and the union bound $\adv$ has to succeed on $\mathcal{D}_1$ with probability of success $s$ that satisfies:
\[\rho \cdot s + (1-\rho) \geq 0.99  \text{,}\]
or, equivalently, \[s \geq \frac{1}{\rho} \left(0.99 - (1 - \rho) \right).\] By Assumption~\ref{item:secondprop} and~\ref{item:thirdprop}, this implies 
\begin{equation}\label{eq:succeesproblog1delta}
s \geq 0.97 \text{.}
\end{equation}
Now observe:
\begin{align}
&\mathbb{E}_{S \sim \mathcal{D}_1}[\mu(p^{-1}(E(ALG(S))))] \nonumber \\
&= \int_{\text{supp}(\mathcal{D})} \mathbb{P}_{S \sim \mathcal{D}_1}[p(x) \in E(ALG(S))] \ d\mu \nonumber \\
&= \int_{\text{supp}(\mathcal{D})} \mathbb{P}_{S \sim \mathcal{D}^m} \left[p(x) \in E(ALG(S)) | \left( AR(ALG(S),\e) \geq \eta \wedge R(ALG(S),\e) \leq \delta \right) \right] \ d\mu 
\nonumber \\
&= \int_{\text{supp}(\mathcal{D})} \frac{ \mathbb{P}_{S \sim \mathcal{D}^m} \left[p(x) \in E(ALG(S)) \cap AR(ALG(S),\e) \geq \eta \cap R(ALG(S),\e) \leq \delta \right]}{\mathbb{P}_{S \sim \mathcal{D}^m}[AR(ALG(S),\e) \geq \eta \wedge R(ALG(S),\e) \leq \delta]} \ d\mu  
\nonumber \\
&\leq \int_{\text{supp}(\mathcal{D})} \frac{ \mathbb{P}_{S \sim \mathcal{D}^m} \left[p(x) \in E(ALG(S)) \right]}{\mathbb{P}_{S \sim \mathcal{D}^m}[AR(ALG(S),\e) \geq \eta \wedge R(ALG(S),\e) \leq \delta]} \ d\mu  
\nonumber \\
&\leq (C \cdot \delta) / \rho  
\nonumber \\
&\leq \frac{1}{0.98} \cdot C \cdot \delta \text{,} 
\label{eq:expbnd}
\end{align}
where the second equality follows from the definition of $\mathcal{D}_1$, third equality follows from the definition of conditioning, first inequality follows from the fact that intersection decreases probability, second inequality is a result of Assumption~\ref{item:firstprop} (which can be applied as $p(x) \in \text{supp}(\mathcal{D}) + B_\e$) and the last inequality is obtained by Assumptions~\ref{item:secondprop}, \ref{item:thirdprop} and the union bound.
Using \eqref{eq:expbnd} we get:
\begin{align}
&\mathbb{P}_{S \sim \mathcal{D}_1} \left[ \mu(p^{-1}(E(ALG(S)))) \geq \frac{AR(ALG(S),\e)}{2} \right] \nonumber \\
&\leq \frac{2 \cdot \mathbb{E}_{S \sim \mathcal{D}_1}[\mu(p^{-1}(E(ALG(S))))]}{AR(ALG(S),\e)}  && \text{by Markov inequality} \nonumber \\
&\leq \frac{2 \cdot \frac{1}{0.98} \cdot C \cdot \delta}{\eta} && \text{by \eqref{eq:expbnd} and definition of } \mathcal{D}_1 \label{eq:log1deltalast}
\end{align}
Applying Theorem~\ref{thm:reduction} to \eqref{eq:succeesproblog1delta} and \eqref{eq:log1deltalast} we get that:
\[QC(ALG,T,m,\e) \geq \log \left(\frac{0.97 \cdot 0.98 \cdot \eta}{2 \cdot C \cdot \delta} \right) \geq \log \left( \frac{\eta}{3 \cdot C \cdot \delta} \right) \text{.}\]
\end{proof}

\newpage
\section{Omitted Proofs - K-NN}
\label{sec:proofsknn}

\thmknn*

\begin{proof}
For $x \in L_- \cup L_+$ and $\rho \in \R$ we will use $x + \rho$ to denote $x + (\rho,0)$. Finally, for $x \in L_- \cup L_+$ we will use $g(x)$ to denote the closest point to $x$ in the other interval. 

\paragraph{Data generation process.} Instead of letting $S \sim \mathcal{D}^{2 m}$ we will use a standard trick and employ a Poisson sampling scheme. This will simplify our proof considerably. Specifically, we think of the samples as being generated by two Poisson processes: Let $N_-$ be a homogeneous Poisson process on the line defined by the extension of $L_-$ and $N_+$ be a independent of $N_-$ homogeneous Poisson process on the line defined by the extension of $L_+$, both of rate $\lambda = 1$. Then we define $A_- := ([0,m)\times \{0\}) \cap N_-, A_+ := ([0,m) \times \{z\}) \cap N_+$ and finally: 
\[S := \{ (x,-1) : x \in A_- \} \cup \{ (x,+1) : x \in A_+ \} \text{ and}\]
\[\widetilde{S} := \{ (x,-1) : x \in N_- \} \cup \{ (x,+1) : x \in N_+ \} \text{.}\]
By design we have $\mathbb{E}[|S|] = 2m$ as $|S|$ is distributed according to $\text{Pois}(2m)$. Moreover, using a standard tail bound for a Poisson random variable, we get that for every $t > 0$:
\begin{equation}\label{eq:poissontailbnd}
\mathbb{P}[||S|- 2m| \geq t] \leq 2e^{-\frac{t^2}{2(2m+t)}} \text{.}
\end{equation}
This means that the size of the dataset generated with the new process is concentrated around $2m$ (with likely deviations of order $\sqrt{m}$).
Let $\{x_1^-,x_2^-, \dots \}$ be the points from $N_-$ with non-negative first coordinate ordered in the increasing order and similarly let $\{x_1^+,x_2^+, \dots \}$. Then note that
$A_- = \{x_1^-, \dots, x_{|A_-|}^- \}$ and $A_+ = \{x_1^+, \dots, x_{|A_+|}^+ \}$. To simplify notation we let $E(S) := E(\text{$1$-Nearest Neighbor}(S)), E(\widetilde{S}) := E(\text{$1$-Nearest Neighbor}(\widetilde{S}))$, where we recall that $E$ denotes the error set. 
Moreover let: 
\[x^-_0 := \max_{x \in N_-, x < 0} x, \ x^+_0 := \max_{x \in N_+, x < 0} x 
\]
We also define the corresponding random variables $\{ X_0^-, X_1^-, \dots \}$ and $\{ X_0^+, X_1^+, \dots \}$, where for every $i$ we have $x_i^- \sim X_i^-$ and $x_i^+ \sim X_i^+$.  

\paragraph{Upper-bounding $\mu(p^{-1}(E(\text{$1$-Nearest Neighbor}(S))))$.}
Let $p$ be a $z/10$-perturbation. 
We analyze only one of the intervals, namely $L_+$, as the situation for $L_-$ is symmetric. For $i \in \mathbb{N}_+ \cup \{0\}$ let $\widetilde{Z}_i$ be a non-negative random variable defined as: 
\[\widetilde{Z}_i := \nu \left(p^{-1} \left(\overline{P}_{x_{i}^+} \cup \overline{P}_{x_{i+1}^+} \right)\right) \text{,}\]
where we define for every $(x,y) \in L_+$:
\[\overline{P}_{(x,y)} := \left\{(x',y') \in \R^2 : y' < \frac{1}{2z}(x'-x)^2 + \frac{z}{2} \right\} \]

Note that by construction:
\begin{equation}\label{eq:erroruprbnd}
\sum_{i=0}^{|A_-|} \widetilde{Z}_i \geq \nu(p^{-1}(E(S)) \cap L
_+) \text{.}
\end{equation}
We divide $\widetilde{Z}_i$'s into $k$ groups, where $k$ will be chosen later. For $i \in \mathbb{N}_+ \cup \{0\}$ we define:
\[\widetilde{Z}_{i/k}^{i \bmod k} := \widetilde{Z}_i \text{.}\] 
Let $g \in \{0, \dots,k-1\}$. We will upper-bound the probability:
\[ \mathbb{P}\left[\sum_{i=0}^{\lceil (1+c(z))m/k \rceil} \widetilde{Z}_i^g \geq \left(1 + \frac{\e(z)}{2}\right) \cdot e^{-\frac{4\sqrt{5}z}{5}} \cdot m \right] \text{,}\]
where the function $c : \R^+ \xrightarrow[]{} \R^+$ will be defined later.

Let $i \in [\lceil (1+c(z))m/k \rceil]$ and $x_0^+,x_1^+, \dots, x_{(i-1)k+g+1}^+ \in \R$ be an increasing sequence such that $x_0^+ < 0 < x_1^+$. Assume that $p$ maximizes $\mathbb{E} \left[\widetilde{Z}_i^g \middle| X_0^+ = x_0^+, X_1^+ = x_1^+, 
\dots, X_{(i-1)k+g+1}^+ = x_{(i-1)k+g+1}^+\right]$. Note that by construction of $\overline{P}$'s we have the following property. For every $i \in [|A_+|]$ and every $(x,y) \in \overline{P}_{x_{i}^+}, \ y \geq 1/2$ we have that for every $y' \in [1/2,y] \ (x,y') \in \overline{P}_{x_{i}^+}$. Using this fact we can assume without loss of generality that for every $t \in L_+$ we have $p(t) = (x,y), y \leq z$ and $\|p(t) - t\|_2 = z/10$. The reason is that if $p(t)$ is above $L_+$ we can flip it with respect to $L_+$ and preserve the distance to $t$ and if $\|p(t) - t\|_2 < z/10$ we can create a new $p'$ that moves $t$ to $p'(t) := (x,y')$, where $(x,y) = p(t), y' < y$ and $\|p'(t) - t\|_2 = z/10$.

For every $t \in L_+$ let:
\[\alpha(t) := \sphericalangle( (-1,0), p(t) - t) \text{.}\]
Now observe that $p(t) \in \overline{P}_{x_{i}^+} \cup \overline{P}_{x_{i+1}^+}$ iff $x_{i}^+ \leq \tau_1$ and $\tau_2 \leq x_{i+1}^+$, where the two threshold can be computed from $p(t)$ or equivalently from $t$ and $\alpha(t)$. We get the following:
\begin{align}
&\mathbb{P}\left[p(t) \in \overline{P}_{x_{i\cdot k+g}^+} \cup \overline{P}_{x_{i\cdot k+g+1}^+} \middle| X_{(i-1) \cdot k + g + 1} = x_{(i-1)k+g+1}^+ \right] \nonumber \\
&= \int_{x_{(i-1)k+g+1}^+}^{\tau_1} f_{X_{i\cdot k+g}^+ - X_{(i-1)k + g + 1}^+}(x' - x_{(i-1)k+g+1}^+) \cdot e^{-(\tau_2 - x')} dx' \nonumber \\
&= e^{-\frac{2\sqrt{5}z}{5} \sqrt{5 - \cos(\alpha(t))}} \cdot \frac{\left[t - \frac{z}{10}(- \sin(\alpha(t)) + 2 \sqrt{5}\sqrt{5 - \cos(\alpha(t))})\right]^k}{k!} \cdot e^{-\left[t -  \frac{z}{10}(- \sin(\alpha(t)) + 2 \sqrt{5}\sqrt{5 - \cos(\alpha(t))}\right]} \nonumber \\
&\leq e^{-\frac{4\sqrt{5}z}{5}} \cdot \frac{\left[t - \frac{z}{10}(- \sin(\alpha(t)) + 2 \sqrt{5}\sqrt{5 - \cos(\alpha(t))})\right]^k}{k!} \cdot e^{-\left[t -  \frac{z}{10}(- \sin(\alpha(t)) + 2 \sqrt{5}\sqrt{5 - \cos(\alpha(t))}\right]} \label{eq:firststepugly}
\end{align}
The first equality follows from the fact that inter-arrival times are independent on $L_+$. To see the second observe that $f_{X_{i\cdot k+g}^+ - X_{(i-1)k + g + 1}^+}$ is the density of Erlang distribution with parameters $(k-1,1)$ and the formula $t -  \frac{z}{10}(- \sin(\alpha(t)) + 2 \sqrt{5}\sqrt{5 - \cos(\alpha(t))}$ gives the expression for $\tau_1$ and $\frac{2\sqrt{5}z}{5} \sqrt{5 - \cos(\alpha(t))} + \tau_1$ gives the expression for $\tau_2$.

Then we have:
\begin{align}
&\mathbb{E} \left[\widetilde{Z}_i^g \middle| X_0^+ = x_0^+, X_1^+ = x_1^+, 
\dots, X_{(i-1)k+g+1}^+ = x_{(i-1)k+g+1}^+\right] \nonumber \\
&= \int_{x_{(i-1)k+g+1}^+}^{\infty} \mathbb{P}_S \left[p(t) \in \overline{P}_{x_{i\cdot k+g}^+} \cup \overline{P}_{x_{i\cdot k+g+1}^+} \middle| X_0^+ = x_0^+, 
\dots, X_{(i-1)k+g+1}^+ = x_{(i-1)k+g+1}^+ \right] dt \nonumber \\
&= \int_{x_{(i-1)k+g+1}^+}^{\infty} \mathbb{P}_S \left[p(t) \in \overline{P}_{x_{i\cdot k+g}^+} \cup \overline{P}_{x_{i\cdot k+g+1}^+} \middle| X_{(i-1)k+g+1}^+ = x_{(i-1)k+g+1}^+ \right] dt \nonumber \\
&\leq e^{-\frac{4\sqrt{5}z}{5}} \int_{0}^{\infty} \frac{\left[t - \frac{z}{10}(- \sin(\alpha(t)) + 2 \sqrt{5}\sqrt{5 - \cos(\alpha(t))})\right]^k}{k!} \cdot e^{-\left[t -  \frac{z}{10}(- \sin(\alpha(t)) + 2 \sqrt{5}\sqrt{5 - \cos(\alpha(t))}\right]} &&\text{By \eqref{eq:firststepugly}} \label{eq:secondstepugly}
\end{align}
Now we bound the expression from \eqref{eq:secondstepugly}. Note that the range of $\sin$ and $\cos$ is $[-1,1]$ so:
\[\left|\frac{z}{10}(- \sin(\alpha(t)) + 2 \sqrt{5}\sqrt{5 - \cos(\alpha(t))})\right| \leq  \frac{11z}{10}\]
Function $e^{-x}  \cdot \frac{x^k}{k!}$ is increasing on $[-\infty,k]$ and decreasing on $[k,\infty]$ thus 
\begin{align}
&\int_{0}^{\infty} \frac{\left[t - \frac{z}{10}(- \sin(\alpha(t)) + 2 \sqrt{5}\sqrt{5 - \cos(\alpha(t))})\right]^k}{k!} \cdot e^{-\left[t -  \frac{z}{10}(- \sin(\alpha(t)) + 2 \sqrt{5}\sqrt{5 - \cos(\alpha(t))}\right]} \nonumber \\
&\leq \int_{0}^{k - \frac{11z}{10}} e^{-(t'+ \frac{11z}{10})} \cdot \frac{(t'+ \frac{11z}{10})^k}{k!} \ dt' + \int_{k-\frac{11z}{10}}^{k+\frac{11z}{10}} e^{-k} \cdot \frac{k^k}{k!} \  dt' + \int_{k+ \frac{11z}{10}} e^{-(t'-\frac{11z}{10})} \cdot \frac{(t'- \frac{11z}{10})^k}{k!} \ dt' \nonumber \\
&\leq \int_0^{\infty} e^{-t'} \cdot \frac{t'^k}{k!} \ dt' + \frac{22z}{10} e^{-k} \cdot \frac{k^k}{k!} \nonumber \\
&\leq 1 + \frac{22z}{10 \sqrt{2\pi k}} \label{eq:oneplussmall}
\end{align}
where the last inequality follows from the fact that the function $e^{-t'} \cdot \frac{t'^k}{k!}$ is the density function of the Erlang distribution with parameters $(k,1)$ and Stirling factorial bounds.

Combining \eqref{eq:secondstepugly} and \eqref{eq:oneplussmall} we get that:
\begin{equation}\label{eq:expectationbnd}
\mathbb{E} \left[\widetilde{Z}_i^g \middle| X_0^+ = x_0^+, X_1^+ = x_1^+, 
\dots, X_{(i-1)k+g+1}^+ = x_{(i-1)k+g+1}^+\right] \leq \left( 1 + \frac{z}{\sqrt{k}} \right)e^{-\frac{4\sqrt{5}z}{5}}\text{.}
\end{equation}

Note that in order for $\widetilde{Z}_i^g \geq 0$ one needs $x_{i\cdot k+g+1}^+ - x_{i\cdot k + g}^+ \geq 2 \cdot \frac{z}{10}(- \sin(\alpha(t)) + 2 \sqrt{5}\sqrt{5 - \cos(\alpha(t))})$, for $\alpha(t) = 0$. Simplifying this is equivalent to $x_{i\cdot k+g+1}^+ - x_{i\cdot k + g}^+ \geq \frac{4\sqrt{5}z}{5}$. As the lengths of intervals are independent we get that for every $i \in \mathbb{N}_+ \cup \{0\}$:  
\begin{equation}\label{eq:zerowithhp}
\mathbb{P}\left[\widetilde{Z}_i^g = 0 | X_0^+ = x_0^+, X_1^+ = x_1^+, 
\dots, X_{(i-1)k+g+1}^+ = x_{(i-1)k+g+1}^+ \right] \geq 1 - e^{-\frac{4\sqrt{5}z}{5}}
\end{equation}
From definition of $\widetilde{Z}_i^g$ we have that $\widetilde{Z}_i^g \leq \nu \left(\left(\overline{P}_{x_{i\cdot k+g}^+} \cup \overline{P}_{x_{i\cdot k+g+1}^+}\right) + B_{z/10} \right)$. We will give an upper bound on $\nu \left(\left(\overline{P}_{x_{i\cdot k+g}^+} \cup \overline{P}_{x_{i\cdot k+g+1}^+}\right) + B_{z/10} \right)$ depending on $x_{i\cdot k+g+1}^+-x_{i\cdot k + g}^+$. For simplicity let $l := x_{i\cdot k+g+1}^+-x_{i\cdot k + g}^+$. Let $\alpha^*$ be the minimizer of $\frac{z}{10}(- \sin(\alpha) + 2 \sqrt{5}\sqrt{5 - \cos(\alpha)})$ and $x^* := \frac{\sqrt{5}\sqrt{5 - \cos(\alpha^*)}z}{5}$. Then for $l \in \left[\frac{4\sqrt{5}z}{5},2x^*\right] $ we have that:
\begin{equation}\label{eq:expressionforlefttail}
\nu \left(\left(\overline{P}_{x_{i\cdot k+g}^+} \cup \overline{P}_{x_{i\cdot k+g+1}^+}\right) + B_{z/10} \right) = 2\sqrt{\frac{z^2}{100} - \left(\frac{z}{2} - \frac{1}{2z}(l/2)^2\right)^2} \text{.}
\end{equation}
For $l \in (2x^*, \infty)$ we have:
\begin{equation}\label{eq:expressionforrighttail}
\nu \left(\left(\overline{P}_{x_{i\cdot k+g}^+} \cup \overline{P}_{x_{i\cdot k+g+1}^+}\right) + B_{z/10} \right) = l - 2x^* + \frac{2z}{10}\sin(\alpha^*) \text{.}
\end{equation}
Thus as the length of the intervals are distributed according to the exponential distribution we get that for $l \in \left[\frac{4\sqrt{5}z}{5},2x^*\right]$:
\begin{equation}\label{eq:smallwithhp}
\mathbb{P} \left[\widetilde{Z}_i^g \geq 2\sqrt{\frac{z^2}{100} - \left(\frac{z}{2} - \frac{1}{2z}(l/2)^2\right)^2} \middle| X_0^+ = x_0^+, X_1^+ = x_1^+, 
\dots, X_{(i-1)k+g+1}^+ = x_{(i-1)k+g+1}^+ \right] \leq e^{-l}
\text{,}
\end{equation}
and similarly for $l \in (2x^*, \infty)$:
\begin{equation}\label{eq:smallwithhp2}
\mathbb{P} \left[\widetilde{Z}_i^g \geq l - 2x^* + \frac{2z}{10}\sin(\alpha^*) \middle| X_0^+ = x_0^+, X_1^+ = x_1^+, 
\dots, X_{(i-1)k+g+1}^+ = x_{(i-1)k+g+1}^+ \right] \leq e^{-l}
\text{,}
\end{equation}

Now we bound the probability that sum of variables from the $g$-th group deviates considerably from its expectation. The idea is to use a method similar to the proof of the Chernoff bound.
\begin{align}
&\mathbb{P}\left[\sum_{i=0 }^{\lceil (1+c(z))m/k \rceil} \widetilde{Z}_i^g \geq \frac{1}{k} \left(1 + \frac{\e(z)}{2}\right)\cdot e^{-\frac{4\sqrt{5}z}{5}} \cdot m \right] \nonumber \\
&\leq \mathbb{P}\left[\widetilde{Z}_0^g \geq \frac{1}{k} \cdot \frac{\e(z)}{4} \cdot e^{-\frac{4\sqrt{5}z}{5}} \cdot m \right] + \mathbb{P}\left[\sum_{i=1 }^{\lceil (1+c(z))m/k \rceil} \widetilde{Z}_i^g \geq \frac{1}{k} \left(1 + \frac{\e(z)}{4}\right)\cdot e^{-\frac{4\sqrt{5}z}{5}} \cdot m \right] && \text{By the union bound} \label{eq:expectationofproduct}
\end{align}
We bound the two terms from \eqref{eq:expectationofproduct} separately. Using \eqref{eq:smallwithhp} we get that for $m \geq \frac{2z}{10}\sin(\alpha^*) \cdot k \cdot \frac{4}{\e(z)} \cdot e^{\frac{4\sqrt{5}z}{5}}$:
\begin{equation}\label{eq:firsttermconditionalbound}
\mathbb{P}\left[\widetilde{Z}_0^g \geq \frac{1}{k} \cdot \frac{\e(z)}{4} \cdot e^{-\frac{4\sqrt{5}z}{5}} \cdot m \right] \leq \text{exp} \left(- \frac{1}{k} \cdot \frac{\e(z)}{4} \cdot e^{-\frac{4\sqrt{5}z}{5}} \cdot m + 2x^* - \frac{2z}{10}\sin(\alpha^*)\right) \text{,}
\end{equation}
which implies:
\begin{equation}\label{eq:firsttermconditionalbound2}
\mathbb{P}\left[\widetilde{Z}_0^g \geq \frac{1}{k} \cdot \frac{\e(z)}{4} \cdot e^{-\frac{4\sqrt{5}z}{5}} \cdot m \right] \leq O_z(1) e^{-\Omega_z(m)}
\end{equation}
Now we bound the second term from \eqref{eq:expectationofproduct}. For every $s>0$::
\begin{align}
&\mathbb{P}\left[\sum_{i=1 }^{\lceil (1+c(z))m/k \rceil} \widetilde{Z}_i^g \geq \frac{1}{k} \left(1 + \frac{\e(z)}{4}\right)\cdot e^{-\frac{4\sqrt{5}z}{5}} \cdot m \right] \nonumber \\
&\leq 
\mathbb{P}\left[\text{exp}\left(s \sum_{i=1}^{\lceil (1+c(z))m/k \rceil} \widetilde{Z}_i^g \right) \geq \text{exp}\left(s \cdot \frac{1}{k} \left(1 + \frac{\e(z)}{4}\right)\cdot e^{-\frac{4\sqrt{5}z}{5}}\cdot m \right) \right] \nonumber \\
&\leq \mathbb{E}\left[\text{exp}\left(s \sum_{i=1}^{\lceil (1+c(z))m/k \rceil} \widetilde{Z}_i^g \right) \right] \cdot \text{exp}\left(-s \cdot \frac{1}{k} \left(1 + \frac{\e(z)}{4}\right)\cdot e^{-\frac{4\sqrt{5}z}{5}} \cdot m \right)  && \text{By Markov inequality} \nonumber \\
&\leq \mathbb{E}\left[ \prod_{i=1}^{\lceil (1+c(z))m/k \rceil} \left[ \text{exp}\left(s \cdot \left( \widetilde{Z}_i^g - \frac{1}{1+c} \left(1 + \frac{\e(z)}{4}\right)\cdot e^{-\frac{4\sqrt{5}z}{5}}\right) \right) \right] \right] \label{eq:productexpandedfirststep}
\end{align}
Set $c(z) := \frac{1+\frac{\e(z)}{4}}{{1+\frac{\e(z)}{8}}} - 1$.
Using the chain rule we obtain:
\begin{align}
&\mathbb{E}\left[ \prod_{i=1}^{\lceil (1+c(z))m/k \rceil} \left[ \text{exp}\left(s \cdot \left( \widetilde{Z}_i^g - \left(1 + \frac{\e(z)}{8}\right)\cdot e^{-\frac{4\sqrt{5}z}{5}}\right) \right) \right] \right] \nonumber \\
& \mathbb{E}\left[ \prod_{i=1}^{\lceil (1+c(z))m/k \rceil - 1} \left[ \text{exp}\left(s \cdot \left( \widetilde{Z}_i^g - \left(1 + \frac{\e(z)}{8}\right)\cdot e^{-\frac{4\sqrt{5}z}{5}}\right) \right) \right] \cdot \mathbb{E}\left[\text{exp}\left(s \cdot \left( \widetilde{Z}_i^g - \left(1 + \frac{\e(z)}{8}\right)\cdot e^{-\frac{4\sqrt{5}z}{5}}\right) \right) \middle| \{\widetilde{Z}_i^g\}_{i=1}^{\lceil 2m/k \rceil -1} \right] \right] \label{eq:productexpanded}
\end{align}

Using the fact that variables $X_0^+, \dots, X_{(i-1)k+g+1}^+$ determine values of $\widetilde{Z}_0^g \dots, \widetilde{Z}_{i-1}^g$ and the bound from \eqref{eq:expectationbnd}
holds for all possible realizations of $X_0^+, \dots, X_{(i-1)k+g+1}^+$ if we maximize the inner conditional expectation of \eqref{eq:productexpanded} over variables $\widetilde{Z}_i^g$ satisfying property \eqref{eq:expectationbnd} we can get an upper bound on $\mathbb{P}\left[\sum_{i=0 }^{\lceil (1+c)m/k \rceil} \widetilde{Z}_i^g \geq \frac{1}{k} \left(1 + \frac{\e(z)}{2}\right)\cdot e^{-\frac{4\sqrt{5}z}{5}} \cdot m \right]$ via \eqref{eq:expectationofproduct}, \eqref{eq:firsttermconditionalbound} and \eqref{eq:productexpandedfirststep}. More formally let's consider a family of random variables $Z$ satisfying:

\begin{enumerate}
    \item $Z \geq 0$ \label{nonegative}
    \item $\mathbb{E} [Z] \leq (1 + \frac{z}{\sqrt{k}}) \cdot e^{-\frac{4\sqrt{5}z}{5}}$, \label{prop1}
    \item For $l \in \left[\frac{4\sqrt{5}z}{5},2x^*\right]$:
    $
    \mathbb{P} \left[Z \geq 2\sqrt{\frac{z^2}{100} - \left(\frac{z}{2} - \frac{1}{2z}(l/2)^2\right)^2} \right] \leq e^{-l}
    \text{,}
    $
    \label{prop2}
    \item For $l \in (2x^*, \infty)$:
    $
    \mathbb{P} \left[Z \geq l - 2x^* + \frac{2z}{10}\sin(\alpha^*) \right] \leq e^{-l}
    \text{.}
    $ \label{prop3} 
\end{enumerate}
Consider the following optimization problem.
\begin{equation}\label{eq:goal}
\sup_{Z : Z \text{ satisfies~\ref{nonegative}, \ref{prop1}, \ref{prop2} and \ref{prop3}}} \mathbb{E} \left[\text{exp} \left(s \cdot \left(Z - \left(1 + \frac{\e(z)}{8}\right)\cdot e^{-\frac{4\sqrt{5}z}{5}} \right) \right) \right]  \text{.}
\end{equation}
The supremum of this problem is attained for some $Z^*$ from the family. This is the case as Properties~\ref{prop2} and \ref{prop3} guarantee that the objective function is bounded. Set $k := \frac{256z^2}{\e(z)^2}$. Observe then that because of Property~\ref{prop1} we have that $\mathbb{E} \left[Z^* - \left(1 + \frac{\e(z)}{8}\right)\cdot e^{-\frac{4\sqrt{5}z}{5}} \right] < 0$. Taylor expanding the function $e^{sX}$ we get that:
\[
\mathbb{E} \left[\text{exp} \left(s \cdot \left(Z^* - \left(1 + \frac{\e(z)}{8}\right)\cdot e^{-\frac{4\sqrt{5}z}{5}} \right) \right) \right] = 1 +s \cdot \mathbb{E} \left[Z^* - \left(1 + \frac{\e(z)}{8}\right)\cdot e^{-\frac{4\sqrt{5}z}{5}} \right] + o(s^2).
\]
Thus we get that there exists $s^* > 0$
such that:
\begin{equation}\label{eq:maximizationproblme}
\sup_{Z : Z \text{ satisfies~\ref{nonegative}, \ref{prop1}, \ref{prop2} and \ref{prop3}}} \mathbb{E} \left[\text{exp} \left(s^* \cdot \left(Z - \left(1 + \frac{\e(z)}{8}\right)\cdot e^{-\frac{4\sqrt{5}z}{5}} \right) \right) \right] < e^{-\Omega_z(1)}\text{.}
\end{equation}
So combining \eqref{eq:firsttermconditionalbound2}, \eqref{eq:productexpanded} and \eqref{eq:maximizationproblme} we get that 
:
\begin{align}
&\mathbb{P}\left[\sum_{i=0 }^{\lceil (1+c)m/k \rceil} \widetilde{Z}_i^g \geq \frac{1}{k} \left(1 + \frac{\e(z)}{2}\right)\cdot e^{-\frac{4\sqrt{5}z}{5}} \cdot m \right] \nonumber \\
&\leq O_z(1) e^{-\Omega_z(m)} + e^{-\Omega_z((1+c(z))m/k)} \nonumber \\
&\leq O_z(1) e^{-\Omega_z(m)} && \text{As $k$ is a function of $z$} \label{eq:singlegroupbound}
\end{align}

Thus we get that: 
\begin{align}
&\mathbb{P}\left[\nu(p^{-1}(E(S)) \cap L
_-) \geq \left(1 + \frac{\e(z)}{2}\right)\cdot e^{-\frac{4\sqrt{5}z}{5}} \cdot m \right] \nonumber \\
&\leq \mathbb{P}\left[\sum_{i=0}^{|A_-|} \widetilde{Z}_i \geq \left(1 + \frac{\e(z)}{2}\right)\cdot e^{-\frac{4\sqrt{5}z}{5}} \cdot m \right] &&\text{By \eqref{eq:erroruprbnd}} \nonumber \\
&\leq \mathbb{P}\left[\left( \sum_{i=0}^{\lceil (1+c(z))m \rceil} \widetilde{Z}_i \geq \left(1 + \frac{\e(z)}{2}\right)\cdot e^{-\frac{4\sqrt{5}z}{5}} \cdot m \right) \vee \left(|A_-| > (1+c(z))m \right)\right] \nonumber \\
&\leq \mathbb{P}\left[\sum_{i=0}^{\lceil(1+c(z))m \rceil} \widetilde{Z}_i \geq \left(1 + \frac{\e(z)}{2}\right)\cdot e^{-\frac{4\sqrt{5}z}{5}} \cdot m \right] +  2e^{-\frac{(1+c(z))^2m^2}{2(m+(1+c(z))m)}} &&\text{Union bound and \eqref{eq:poissontailbnd}} \nonumber \\
&\leq  k \cdot O_z(1) e^{-\Omega_z(m)} + 2e^{-\Omega_z(m)} &&\text{By \eqref{eq:singlegroupbound} and union bound over groups} \nonumber \\
&\leq \frac{256z^2}{\e(z)^2} \cdot O_z(1) e^{-\Omega_z(m)} + 2e^{-\Omega_z(m)} &&\text{By setting of $k$} \nonumber \\
&\leq O_z(1) e^{-\Omega_z(m)} \nonumber
\end{align}
The above fact together with the union bound over $L_-$ and $L_+$ gives:
\begin{equation}\label{eq:pmovesatmost}
\mathbb{P}[\mu(p^{-1}(E(\text{$1$-Nearest Neighbor}(S)))) \geq \left(1 + \frac{\e(z)}{2}\right)\cdot e^{-\frac{4\sqrt{5}z}{5}}] \leq O_z(1)e^{-\Omega_z(m)}    
\end{equation}

\paragraph{Lower-bounding $AR(\text{$1$-Nearest Neighbor}(S), z/10)$.} We will focus on $L_+$ as the argument for $L_-$ is analogous. Let $a = (a_1,z),b = (b_1,z) \in A_+$ be two consecutive points from $A_+$. Note that a parabola defined by:
\[P_a := \left\{ \left(t + a_1,\frac{1}{2z}t^2 + \frac{z}{2}\right) : t \in \R \right\} 
\text{,}\]
is exactly the set of points that are equally distant to $a$ and $L_-$. An analogous parabola can be defined for the point $b$. Let $P_a^{\xrightarrow{}}$ be the parabola $P_a$ shifted to the right by $\rho$ (to be fixed later). Formally:
\[P_a^{\xrightarrow{}} := \left\{ \left(t + a_1,\frac{1}{2z}(t-\rho)^2 + \frac{z}{2}\right) : t \in \R \right\} \text{.}\]
Similarly let:
\[P_b^{\xleftarrow{}} := \left\{ \left(t + b_1,\frac{1}{2z}(t+\rho)^2 + \frac{z}{2}\right) : t \in \R \right\} \text{.}\]

We will show if a point $(x,y) \in \R^2$ is below $P_a^{\xrightarrow{}},P_b^{\xleftarrow{}}$  and $y \leq 2z, x \in (a_1,b_1)$ then $(x,y) \in E(S)$ with high probability. 
More precisely let $(x,y) \in \R^2$ be such that: $x \in (a_1,b_1), y \leq \frac{1}{2z}(x-a_1-\rho)^2 + \frac{z}{2}, y \leq \frac{1}{2z}(x-b_1+\rho)^2 + \frac{z}{2}, y \in \left[\frac{9z}{10},2z\right]$. By construction $d((x,y), A_+)$ is obtained at $a$ or $b$. We have that:
\begin{align}
d(a, (x,y))^2 
&\geq (x-a_1)^2 + \left(\frac{z}{2} - \frac{(x-a_1-\rho)^2}{2z}\right)^2 \nonumber \\
&= (x-a_1)^2 + \frac{(x-a_1-\rho)^4}{4z^2} - \frac{(x-a_1-\rho)^2}{2} + \frac{z^2}{4} \label{eq:daxy}
\end{align}
\begin{align}
d(L_-,(x,y))^2 
&\leq  \left(\frac{(x-a_1-\rho)^2}{2z}+\frac{z}{2}\right)^2 \nonumber \\ 
&= \frac{(x-a_1-\rho)^4}{4z^2} + \frac{(x-a_1-\rho)^2}{2} + \frac{z^2}{4} \label{eq:dlminxy2}
\end{align}
Now if $d(a,(x,y)) > 3z$ then $d(a, (x,y)) -  d(L_-,(x,y)) \geq z$ by assumption that $y \leq 2z$. Otherwise we have:
\begin{align}
d(a, (x,y)) -  d(L_-,(x,y))
&= \frac{d(a, (x,y))^2 -  d(L_-,(x,y))^2}{d(a, (x,y)) +  d(L_-,(x,y))} \nonumber \\
&\geq \frac{\rho(2x-2a_1-\rho)}{3z + 2z} && \text{By \eqref{eq:daxy}, \eqref{eq:dlminxy2}}  \nonumber \\
&\geq \frac{\rho \left(\frac{4z}{\sqrt{5}} - \rho \right)}{5z} && \text{As $y \geq 0.9z$} \nonumber \\
&\geq 0.3 \rho && \text{As $z > 10\rho$} \label{eq:distlwrbnd}
\end{align}
By symmetry an analogous bound holds for $d(b, (x,y)) -  d(L_-,(x,y))$.


Observe that if there exist a point $c \in A_-$ such that $c \in [(x - \sqrt{0.1\rho z},0), (x + \sqrt{0.1\rho z},0) ]$ then $d(a,(x,y)) > d(L_-,(x,y))$. That's true because: 
\begin{align}
d(c,(x,y)) 
&\leq \sqrt{y^2 + 0.1 \rho z} \nonumber \\
&\leq y \sqrt{1 + \frac{0.1 \rho z}{y^2}} \nonumber \\
&\leq y \sqrt{1 + \frac{0.13 \rho}{z}} && \text{As $y \geq 0.9z$} \nonumber \\
&\leq y\left(1 + \frac{0.07 \rho}{z}\right) \nonumber \\
&\leq y + 0.14 \rho && \text{As $y \leq 2z$}\text{,} \label{eq:dcxy} 
\end{align}
Noticing that $y = d(L_-, (x,y))$ we get:
\begin{align}
d(a,(x,y)) - d(c,(x,y))  
&\geq d(a,(x,y)) - d(L_-,(x,y)) - 0.14\rho && \text{By \eqref{eq:dcxy}} \nonumber \\
&\geq 0.3\rho - 0.14\rho \nonumber \\
&> 0 \label{eq:xymisclassified}
\end{align}
By symmetry we also get that $d(b,(x,y)) > d(L_-,(x,y))$, which also implies that $(x,y) \in E(S)$.
Note that:
\[ 
|N_- \cap [(x - \sqrt{0.1\rho z},0), (x + \sqrt{0.1 \rho z},0) ]| \sim \text{Pois}(2\sqrt{0.1 \rho z}) \text{,}
\]
so $\mathbb{P}[N_- \cap [(x - \sqrt{0.1 \rho z},0), (x + \sqrt{0.1 \rho z},0) ] \neq \emptyset] = 1 - e^{-2\sqrt{0.1 \rho z}} \geq 1 - e^{-0.6\sqrt{\rho z}}$, which gives:
\begin{equation}\label{eq:probinerrorset}
\mathbb{P}[(x,y) \in E(S)] \geq    1 - e^{-0.6\sqrt{\rho z}} 
\end{equation}


For $i \in \N_+$ let $\widetilde{Y}_i$ be the random variable defined as: 
\[\widetilde{Y}_i := \nu((E(\widetilde{S}) + B_{z/10}) \cap [x_i^+,x_{i+1}^+)) \text{,}\]
where $\nu$ is one dimensional Lebesgue measure on $L_+$. In words, $\widetilde{Y}_i$ is the random variable that is equal to how much the interval $[x_i^+,x_{i+1}^+)$ contributes to $AR(\text{$1$-Nearest Neighbor}(\widetilde{S}), z/10)$. Observe that $\widetilde{Y}_i$ is primarily determined by the length of $[x_i^+,x_{i+1}^+)$ as well as where the points of $N_-$ are located with respect to $[x_i^+,x_{i+1}^+)$.  

$\widetilde{Y}_i$'s satisfy the following properties:

\begin{enumerate}
    \item $\widetilde{Y}_i$ is non-negative, \label{prop:nonneg}
    \item For all $l \in \left[\frac{4\sqrt{5}z}{5},2x^*\right]$:
    $
    \mathbb{P} \left[\widetilde{Y}_i \geq 2\sqrt{\frac{z^2}{100} - \left(\frac{z}{2} - \frac{1}{2z}(l/2)^2\right)^2} \right] \geq e^{-l-2\rho} \cdot (1 - 2e^{-0.6\sqrt{\rho z }})
    \text{,} \label{prop:taillwrbnd}
    $
    \item For all $l \in (2x^*, \infty)$:
    $
    \mathbb{P} \left[\widetilde{Y}_i \geq l - 2x^* + \frac{2z}{10}\sin(\alpha^*) \right] \geq e^{-l-2\rho} \cdot (1 - 2e^{-0.6\sqrt{\rho z }})
    \text{,}    
    $
    \item $\widetilde{Y}_i$'s are i.i.d. \text{.} 
\end{enumerate}
The first property (non-negativity) is true by definition. 
To see the second and the third (observe similarity to \eqref{eq:smallwithhp} and \eqref{eq:smallwithhp2}) consider $P_{x_i^+}^{\xrightarrow{}},P_{x_{i+1}^+}^{\xleftarrow{}}$ and define $\overline{P}_{x_i^+}^{\xrightarrow{}}$ to be all the points below $P_{x_i^+}^{\xrightarrow{}}$ and $\overline{P}_{x_{i+1}^+}^{\xleftarrow{}}$ analogously. Note that:
\[ 
\nu \left(\left(\left(\overline{P}_{x_i^+}^{\xrightarrow{}} \cup \overline{P}_{x_{i+1}^+}^{\xleftarrow{}}\right) \cap E(S) \right) + B_{z/10}\right) \leq \widetilde{Y}_i
\]
Moreover:
\[
\nu \left(\left(\left(\overline{P}_{x_i^+}^{\xrightarrow{}} \cup \overline{P}_{x_{i+1}^+}^{\xleftarrow{}}\right) \cap E(S) \right) + B_{z/10}\right) = \nu \left(\left(\overline{P}_{x_i^+}^{\xrightarrow{}} \cup \overline{P}_{x_{i+1}^+}^{\xleftarrow{}}\right)  + B_{z/10}\right) \text{,}
\]
as for the equality to hold it is enough for $E(S)$ to contain an interval $[(x,y), (x',y)] \subseteq \overline{P}_{x_i^+}^{\xrightarrow{}} \cup \overline{P}_{x_{i+1}^+}^{\xleftarrow{}}$ that certifies $\nu \left(\left(\overline{P}_{x_i^+}^{\xrightarrow{}} \cup \overline{P}_{x_{i+1}^+}^{\xleftarrow{}}\right)  + B_{z/10}\right)$. This happens if $(x,y),(x',y) \in E(S)$ as then $[(x,y), (x',y)]$ by construction. Finally $(x,y),(x',y) \in E(S)$ with probability at least $1 - 2e^{-0.6\sqrt{\rho z }}$ by \eqref{eq:probinerrorset} and the union bound. Properties two and three follow by observing that $\nu \left(\left(\overline{P}_{x_i^+}^{\xrightarrow{}} \cup \overline{P}_{x_{i+1}^+}^{\xleftarrow{}}\right)  + B_{z/10}\right)$ was already computed in \eqref{eq:expressionforlefttail} and \eqref{eq:expressionforrighttail}. 
The last property is in turn a consequence of the fact that the inter-arrival times of a Poisson process are i.i.d. and that the points on the ``other'' line are Poisson as well and independent of the first line.

Using these properties we have that for every $i \in \N_+$:
\begin{align*}
\mathbb{E}[\widetilde{Y}_i] 
&= \int_0^{\infty} \mathbb{P}[\widetilde{Y}_i > t] dt \\
&\geq (1 - 2e^{-0.6\sqrt{\rho z }}) \cdot \left( \int_{0}^{\frac{2z}{10} \sin(\alpha^*)} e^{-\frac{2\sqrt{5z^2-\sqrt{z^4-25z^2t^2}}}{\sqrt{5}}-2\rho} dt + \int_{\frac{2z}{10} \sin(\alpha^*)}^{\infty} e^{-t - 2x^* + \frac{2z}{10}\sin(\alpha^*)-2\rho} dt \right) \\
&= (1 - 2e^{-0.6\sqrt{\rho z }}) \cdot e^{-2\rho} \cdot \left( \int_{0}^{\frac{2z}{10} \sin(\alpha^*)} e^{-\frac{2\sqrt{5z^2-\sqrt{z^4-25z^2t^2}}}{\sqrt{5}}} dt + \int_{\frac{2z}{10} \sin(\alpha^*)}^{\infty} e^{-t - 2x^* + \frac{2z}{10}\sin(\alpha^*)} dt \right)
\end{align*}

Our goal now is to show that $\sum_{i=1}^{(1-c'(z))m} \widetilde{Y}_i \geq  \left(1 + \frac{2\e(z)}{3}\right) \cdot e^{-\frac{4\sqrt{5}z}{5}} \cdot m$ with high probability, where the function $c': \R^+ \xrightarrow[]{} \R^+$ will be defined later. Similarly to the standard proof of the Chernoff bound, for every $s > 0$:
\begin{align}
&\mathbb{P} \left[\sum_{i=1}^{(1- c'(z))m} \widetilde{Y}_i \leq \left(1 + \frac{2\e(z)}{3}\right) \cdot e^{-\frac{4\sqrt{5}z}{5}} \cdot m \right] \nonumber \\
&= \mathbb{P} \left[\text{exp}\left(-s \sum_{i=1}^{(1- c'(z))m} \widetilde{Y}_i \right) \geq \text{exp}\left(-s \cdot \left(1 + \frac{2\e(z)}{3}\right) \cdot e^{-\frac{4\sqrt{5}z}{5}} \cdot m\right) \right] \nonumber \\
&\leq \mathbb{E} \left[\text{exp}\left(-s \sum_{i=1}^{(1- c'(z))m} \widetilde{Y}_i\right) \right] \cdot \text{exp}\left(s \cdot \left(1 + \frac{2\e(z)}{3}\right) \cdot e^{-\frac{4\sqrt{5}z}{5}} \cdot m \right)  && \text{by Markov inequality}\nonumber \\
&= \left(\mathbb{E}\left[\text{exp}\left(-s\widetilde{Y}_1\right) \right] \cdot \text{exp}\left(s \cdot \frac{1}{1-c'(z)} \left(1 + \frac{2\e(z)}{3}\right) \cdot e^{-\frac{4\sqrt{5}z}{5}} \cdot m\right) \right)^{(1-c'(z))m} && \text{as $\widetilde{Y}_i$'s are i.i.d.} \label{eq:chernofflowertail}
\end{align}
Set $c'(z) := 1 - \frac{1+ \frac{2\e(z)}{3}}{1+\frac{3\e(z)}{4}}$. Then the above becomes:
\[
\left(\mathbb{E}\left[\text{exp}\left(s\left(-\widetilde{Y}_1 + \left(1 + \frac{3\e(z)}{4}\right) \cdot e^{-\frac{4\sqrt{5}z}{5}}  \right)\right) \right]  \right)^{(1-c'(z))m}
\]
Taylor expanding the function $e^{sX}$ we get that:
\begin{equation}\label{eq:taylorexp}
\mathbb{E} \left[\text{exp}\left(s\left(-\widetilde{Y}_1 + \left(1 + \frac{3\e(z)}{4}\right) \cdot e^{-\frac{4\sqrt{5}z}{5}}  \right)\right) \right] = 1 + s \cdot \mathbb{E} \left[-\widetilde{Y}_1 + \left(1 + \frac{3\e(z)}{4}\right) \cdot e^{-\frac{4\sqrt{5}z}{5}} \right] + o(s^2).
\end{equation}
We will show now that there exists a function $\e : \R^+ \xrightarrow[]{} \R^+$ such that for $z$ bigger than a constant: 
\[ 
\mathbb{E}[\widetilde{Y}_i] \geq (1 + \e(z)) \cdot e^{-\frac{4\sqrt{5}z}{5}} \text{.}
\]
Note that it is equivalent to showing that for $z$ bigger than a constant
\[
\max_{z/10 > \rho > 0} (1 - 2e^{-0.6\sqrt{\rho z }}) \cdot e^{-2\rho} \cdot \left( \int_{0}^{\frac{2z}{10} \sin(\alpha^*)} e^{-\frac{2\sqrt{5z^2-\sqrt{z^4-25z^2t^2}}}{\sqrt{5}}} dt + \int_{\frac{2z}{10} \sin(\alpha^*)}^{\infty} e^{-t - 2x^* + \frac{2z}{10}\sin(\alpha^*)} dt \right) > e^{-\frac{4\sqrt{5}z}{5}} \text{.}
\]
First observe that $(1 - 2e^{-0.6\sqrt{\rho z }}) \cdot e^{-2\rho}$ can be made arbitrarily close to $1$ by setting $\rho := z^{-1/2}$. Next we lower bound the first integral:
\begin{align*}
& \int_{0}^{\frac{2z}{10} \sin(\alpha^*)} e^{-\frac{2\sqrt{5z^2-\sqrt{z^4-25z^2t^2}}}{\sqrt{5}}} dt \\
&\geq \int_{0}^{\sqrt{\frac{z}{5\sqrt{5}}}} e^{-\frac{2\sqrt{5z^2-\sqrt{z^4-25z^2t^2}}}{\sqrt{5}}} dt  &&\text{For $z$ big enough, as $\sin(\alpha^*)$ is a constant} \\
&\geq e^{-\frac{4\sqrt{5}z}{5}} \cdot \int_{0}^{\sqrt{\frac{z}{5\sqrt{5}}}} e^{-\frac{10\sqrt{5}}{4z} t^2} dt &&\text{As $\frac{2\sqrt{5z^2-\sqrt{z^4-25z^2t^2}}}{\sqrt{5}} \leq \frac{4\sqrt{5}z}{5} + \frac{10\sqrt{5}}{4z}t^2$ for $t \leq z/5$} \\
&\geq e^{-\frac{4\sqrt{5}z}{5}} \cdot \sqrt{\frac{ z}{5\sqrt{5}}} \cdot \frac{\sqrt{\pi}}{2} \text{erf}(1/\sqrt{2}) \text{,}
\end{align*}
thus for $z$ big enough we get that $\mathbb{E}[\widetilde{Y}_i] \geq (1 + \e(z)) \cdot e^{-\frac{4\sqrt{5}z}{5}} $. Using \eqref{eq:taylorexp} this implies that there exists $s^* > 0$ such that $\mathbb{E} \left[\text{exp}\left(s^* \cdot \left(-\widetilde{Y}_1 + \left(1 + \frac{3\e(z)}{4}\right) \cdot e^{-\frac{4\sqrt{5}z}{5}}  \right)\right) \right] < 1 $. Using \eqref{eq:chernofflowertail} we get then that:
\begin{equation}\label{eq:Ytildebound}
\mathbb{P} \left[\sum_{i=1}^{(1- c'(z))m} \widetilde{Y}_i \leq \left(1 + \frac{2\e(z)}{3}\right) \cdot e^{-\frac{4\sqrt{5}z}{5}} \cdot m \right] \leq e^{-\Omega_z(m)}\text{.}
\end{equation}

Now for $i \in [|A_+|-1]$ let $Y_i$ be the random variable defined as: 
\[Y_i := \nu(((E(S) \cap [x_i^+,x_{i+1}^+)) + B_{z/10}) \cap L_-)\text{.}\]
Notice that for all $i \in [|A_+|-1]$ we have $Y_i = \widetilde{Y}_i$. Note that by Poisson tail bound we have:
\begin{equation}\label{eq:numberofpointsbnd} 
\mathbb{P}\left[|A_+| \leq (1-c'(z))m \right] \leq 2e^{-\frac{(1+c'(z))^2m^2}{2(m+(1+c'(z))m)}} \leq 2 e^{-\Omega_z(m)}.
\end{equation}
Combining \eqref{eq:numberofpointsbnd} and \eqref{eq:Ytildebound} and the union bound we get that:
\begin{align*}
\mathbb{P} \left[\sum_{i \in [|A_+|-1]} Y_i \leq \left(1 + \frac{2\e(z)}{3}\right) \cdot e^{-\frac{4\sqrt{5}z}{5}} \cdot m \right]
&\leq \mathbb{P}\left[ (|A_+| \leq (1-c'(z))m) \vee \left(\sum_{i=1}^{(1-c'(z))m} \widetilde{Y}_i \leq \left(1 + \frac{2\e(z)}{3}\right) \cdot e^{-\frac{4\sqrt{5}z}{5}} \cdot m \right) \right] \\
&\leq \mathbb{P}[|A_-| \leq m/2] + \mathbb{P} \left[\sum_{i=1}^{(1-c'(z))m} \widetilde{Y}_i \leq \left(1 + \frac{2\e(z)}{3}\right) \cdot e^{-\frac{4\sqrt{5}z}{5}} \cdot m \right]\\
&\leq 2e^{-\Omega_z(m)} + e^{-\Omega_z(m)} \\
&\leq O_z(1) e^{-\Omega_z(m)} \text{.}
\end{align*}
Note that we omitted the first and the last interval $\left(\left[0,x^-_1\right) \text{ and } \left[x^-_{|A_-|},m\right)\right)$. Omitting these intervals is valid as we are deriving a lower bound for $AR(\text{$1$-Nearest Neighbor}(S), z/10) $. We conclude using the union bound over two intervals $L_-$ and $L_+$ to obtain:
\begin{equation}\label{eq:advrisklwrbnd}
\mathbb{P} \left[AR(\text{$1$-Nearest Neighbor}(S), z)
\leq \left(1 + \frac{2\e(z)}{3}\right) \cdot e^{-\frac{4\sqrt{5}z}{5}} \right] \leq O_z(1) e^{-\Omega_z(m)} \text{.}  
\end{equation}

\paragraph{Lower-bounding $QC$.}
To prove a lower-bound on the QC of $1$-NN applied to this task we will use Theorem~\ref{thm:reduction}. This means that we need to upper-bound:
\[ \sup_{p : \text{ $z/10$-perturbation}} \mathbb{P}_{S} \left[ \mu(p^{-1}(E(\text{$1$-Nearest Neighbor}(S)))) \geq (1-\lambda) \cdot AR(\text{$1$-Nearest Neighbor}(S),z/10) \right] \text{,}\]
where $S$ is generated from the two independent Poisson processes as described at the beginning of the proof. By Remark~\ref{rem:aboutreduction} we can use Theorem~\ref{thm:reduction} in this case. 

Combining \eqref{eq:pmovesatmost} and \eqref{eq:advrisklwrbnd} we get that there exists $\lambda : \R^+ \xrightarrow[]{} \R^+$, which can be set to $\lambda(z) =\frac{1 + \frac{\e(z)}{2}}{1 + \frac{2\e(z)}{3}}$, so that: 
\begin{align*} 
&\sup_{p : \text{ $z/10$-perturbation}} \mathbb{P}_{S} \left[ \mu(p^{-1}(E(S))) \geq (1-\lambda(z)) \cdot AR(\text{$1$-Nearest Neighbor}(S),z/10) \right] \\
&\leq O_z(1)e^{-\Omega_z(m)} + O_z(1) \cdot e^{-\Omega_z(m)} \\
&\leq O_z(1) e^{-\Omega_z(m)}
\text{.} 
\end{align*}

This, by Theorem~\ref{thm:reduction}, means that:
\begin{align*}
QC(\text{$1$-Nearest Neighbor},T_{\text{intervals}},2m,z)
&\geq \Theta \left(\log \left(\frac{0.1}{O_z(1) e^{-\Omega_z(m)}} \right) \right) \\
&\geq \Theta_z(m) \text{.}
\end{align*}

\end{proof}
\newpage
\section{Omitted Proofs - Quadratic Neural Network}
\label{sec:proofsqnn}

We now present proofs of claims from Section~\ref{sec:quadraticnetwork}. Recall that this section deals with quadratic neural nets applied to the concentric spheres dataset.

\subsection{QC lower bounds for exponentially small risk}

We first use the results from Section~\ref{sec:easylowerbound} to argue that increased accuracy leads to an improved guarantee for robustness. We analyze the QC of QNN for $\e = 0.1$, that is $\e$ which is comparable with the separation between the classes. It was experimentally shown in \citet{adversarialSpheres} that increasing the sample size for QNN leads to a higher accuracy on the CS dataset. Thus we assume that for some $m \in \N$ the following holds:
\begin{equation}\label{eq:probofbigerrorbnd}
\mathbb{P}_{S \sim \mathcal{D}^m} [R(\text{QNN}(S)) \in [\delta/2,2\delta]] \geq 1- \delta \text{.}
\end{equation}
Let $S \in (\R^d)^m$ be such that $R(QNN(S)) \geq \delta/2$. We have that $AR(QNN(S),\e) = \mu(E(QNN(S)) + B_\e) \geq \mu((E(QNN(S)) \cap (S_1^{d-1} \cup S_{1.3}^{d-1} ))\ + B_\e)$. Isoperimetric inequality for spheres states that $\mu((E(QNN(S)) \cap (S_1^{d-1} \cup S_{1.3}^{d-1} ))\ + B_\e)$ is maximized when $(E(QNN(S)) \cap (S_1^{d-1} \cup S_{1.3}^{d-1} )$ is a spherical cap of $S_{1.3}^{d-1}$. Using the standard bounds on volumes of spherical caps we get that there exists a universal constant $c > 0$ such that if $\delta \geq 2^{-c \cdot d}$ then $\mu((E(QNN(S)) \cap (S_1^{d-1} \cup S_{1.3}^{d-1} ))\ + B_\e) \geq 1/5$. By \eqref{eq:probofbigerrorbnd} it implies that:
\begin{equation}
\mathbb{P}_{S \sim \mathcal{D}^m}[AR(QNN(S),\e) \geq 1/5] \geq 0.99 \text{.}
\end{equation}

Moreover note that $\mathcal{D}$ is symmetric and thus it is natural to assume that $\mathbb{P}_{S \sim \mathcal{D}^m}[\text{ALG}(S)(x) \neq h(x)]$ only depends on $\|x\|_2$. Using \eqref{eq:probofbigerrorbnd} we bound:
\begin{align*}
&\int_{\text{supp}(\mathcal{D})} \mathbb{P}_{S \sim \mathcal{D}^m}[\text{ALG}(S)(x) \neq h(x)] \ d\mu 
= 
\mathbb{E}_{S \sim \mathcal{D}^m}[R(\text{QNN})] \\
&\leq 2\delta \cdot \mathbb{P}_{S \sim \mathcal{D}^m} [R(\text{QNN}(S)) \leq 2\delta] + 1 \cdot (1 - \mathbb{P}_{S \sim \mathcal{D}^m} [R(\text{QNN}(S)) \leq 2\delta]) 
\leq 3 \delta \text{,}
\end{align*}
which,  assuming that points from $S_1^{d-1}$ are misclassified equally likely as points from $S_{1.3}^{d-1}$ gives that
$\forall x \in \text{supp}(\mathcal{D}), \ \mathbb{P}_{S \sim \mathcal{D}^m}[\text{ALG}(S)(x) \neq h(x)] \leq 3\delta$. Finally we assume that there exists a universal constant $C \in \R$ such that $\forall x \in \text{supp}(\mathcal{D}) + B_\e, \ \mathbb{P}_{S \sim \mathcal{D}^m}[\text{ALG}(S)(x) \neq h(x)] \leq 3 \cdot C \cdot \delta $. This assumption is consistent with the experimental results from \citet{adversarialSpheres} and intuitively it states that points that are $\e$ close to the $\text{supp}(\mathcal{D})$ are at most $C$ times more likely to be misclassified as points from the $\text{supp}(\mathcal{D})$.

Combining the properties and applying Theorem~\ref{thm:log1overdeltabound} we get:
$QC(\text{QNN},\text{CS},m,0.1) \geq \log \left(\frac{1/5}{9 \cdot C \cdot \delta} \right)$, provided that $\delta \geq 2^{-c \cdot d}$.
In words, if QNN has a risk of $2^{-\Omega(k)}$ then it is secure against $\Theta(k)$-bounded adversaries for $\e = 0.1$.

\subsection{QC lower bounds for constant risk}

We give QC lower bounds for the case where the risk achieved by the network is as large as a constant. To get started, let us formally define the distributions and error sets that we will be concerned with. Recall that for $y \in S_1^{d-1}$ we define
$\text{cap}(y,r,\tau) := B_r \cap \{x \in \R^d : \langle x,y \rangle \geq \tau \}$. Let $\tau : [0,1] \xrightarrow{} [0,1]$ be such that for every $\delta \in [0,1]$ we have $\nu(\text{cap}(\cdot,1,\tau(\delta))) / \nu(S_1^{d-1}) =  \delta$, where $\nu$ is a $d-1$ dimensional measure on the sphere $S_1^{d-1}$. Recall that for $k \in \N_+$:
\[E_-(k) = \text{cap}(e_1, 1.15, \tau(\delta/k)) \setminus B_{1.15/1.3} \]
\[E_+(k) = \text{cap}(e_1, 1.495, 1.3\tau(\delta/k)) \setminus B_{1.15} \text{,}\]

\begin{definition}(\textbf{Distributions on Spherical Caps})\label{def:distriboncaps}
\begin{itemize}
    \item $\textbf{Cap}$. Let $\delta \in (0,1)$. We define $\text{Cap}(\delta)$ as a distribution on subsets of $B_{1.15}$ defined by the following process: generate $y \sim U[S_1^{d-1}], b \sim U\{-1,1\}$. Return: $\text{cap}(y_i, 1.15, \tau(\delta/k)) \setminus B_{1.15/1.3}$ if $b = -1$ and $\text{cap}(y, 1.495, 1.3\tau(\delta/k)) \setminus B_{1.15}$ otherwise.
    \item $\textbf{Caps}_k^{\text{i.i.d.}}$. Let $k \in N_+, \delta \in (0,1)$. We define $\text{Caps}_k^{i.i.d.}(\delta)$ as a distribution on subsets of $\R^d$ defined by the following process: generate a sequence of random bits $b_1, \dots, b_k \sim U\{-1,1\}$, generate a sequence of random vectors $y_1, \dots, y_k \sim U[S_1^{d-1}]$. Return: 
    \[\bigcup_{i : b_i = -1} [\text{cap}(y_i, 1.15, \tau(\delta/k)) \setminus B_{1.15/1.3}] \cup \bigcup_{i : b_i = +1} \left[\text{cap}(y_i, 1.495, 1.3\tau(\delta/k)) \setminus B_{1.15}\right] \]
    In words $\text{Caps}_k^{i.i.d.}(\delta)$ generates $k$ i.i.d. randomly rotated sets, each either $E_-(k)$ or $E_+(k)$. 
    \item $\textbf{Caps}_k^{\mathcal{G}}$. Let $k \in N_+, \delta \in (0,1)$, $\mathcal{G}$ be a distribution on $(S_1^{d-1})^k$. We define $\text{Caps}_k^{\mathcal{G}}(\delta)$ as a distribution on subsets of $\R^d$ defined by the following process: generate a sequence of random bits $b_1, \dots, b_k \sim U\{-1,1\}$, generate $y_1, \dots, y_{k} \sim \mathcal{G}$, generate an orthonormal matrix $M \sim O(d)$. Return: 
    \[\bigcup_{i : b_i = -1} M(\text{cap}(y_i, 1.15, \tau(\delta/k)) \setminus B_{1.15/1.3}) \cup \bigcup_{i : b_i = +1} M(\text{cap}(y_i, 1.495, 1.3\tau(\delta/k)) \setminus B_{1.15})  \]
    In words $\text{Caps}_k^{\mathcal{G}}(\delta)$ generates $k$ randomly rotated sets, each either $E_-(k)$ or $E_+(k)$, where relative positions of normal vectors of the sets are defined by $\mathcal{G}$.
\end{itemize}
\end{definition}

Note that definitions of $\text{Cap}, \text{Caps}_k^{i.i.d.}$ and $\text{Caps}_k^{\mathcal{G}}$ are compatible in the following sense:
\begin{observation}
For every $k \in \N_+, \delta \in (0,1)$:
\begin{itemize}
    \item $\text{Caps}_1^{i.i.d.}(\delta) = \text{Cap}(\delta)$ \text{,}
    \item $\text{Caps}_k^{i.i.d.}(\delta) = \text{Caps}_k^{U[(S_1^{d-1})^k]}(\delta)$ \text{.}
\end{itemize}
\end{observation}

In the following lemma we show a reduction from $\text{Cap}_k^{\text{i.i.d}}$ to $\text{Cap}$. This means that we show that if there is an adversary that uses $q$ queries and succeeds on $\text{Cap}_k^{\text{i.i.d}}$ then there exists an adversary that succeeds on $\text{Cap}$ and also asks at most $q$ queries. The takeaway from this lemma is that the QC of $\text{Cap}_k^{\text{i.i.d}}$ is no smaller than the QC of $\text{Cap}$. Formally: 


\begin{lemma}[\textbf{Reduction from $\text{Caps}_k^{\text{i.i.d.}}$ to $\text{Cap}$}]
\label{lem:reductioniid}
Let $k \in \N_+$. If there exists a $q$-bounded adversary $\adv$ that succeeds on $\text{Caps}_k^{\text{i.i.d.}}(0.01)$ with approximation constant $1/2$, error probability $0.01$ and $\e = \tau(0.01/k)$
then there exists a $q$-bounded adversary $\adv'$ that succeeds on $\text{Cap}(0.01/k)$ with approximation constant $\frac{1}{2k}$, error probability of $1 - \frac{1}{3k}$ and the same $\e$.
\end{lemma}

\begin{proof}
Algorithm~\ref{alg:emulaterand} invoked with $\delta = 0.01$ defines the protocol for $\adv'$. We will show that this protocol satisfies the statement of the Lemma.

\begin{algorithm}[H]
\caption{\textsc{EmulateIID}($f, \adv, \delta, k$) \hfill $\triangleright$ $f$ is the attacked classifier \newline  \text{ } \hfill $\triangleright$  $\adv$ is an adversary for distribution $\text{Cap}_k^{\text{i.i.d.}}(\delta)$ }
\label{alg:emulaterand}
\begin{algorithmic}[1]
    \STATE $y_1, \dots, y_{k-1} \sim U[S_1^{d-1}]$
    \STATE $b, b_1, \dots, b_{k-1} \sim U\{-1,1\}$ 
    \FOR{i = 1, \dots, k-1}
        \STATE $C_i := \begin{cases} \text{cap}(y_i, 1.15, \tau(\delta/k)) \setminus B_{1.15/1.3} & \mbox{if } b_i = -1 \\ \text{cap}(y_i, 1.4, 1.3\tau(\delta/k)) \setminus B_{1.15} & \mbox{if } b_i = +1 \end{cases}$
    \ENDFOR
    \STATE $p := $ Simulate $\adv$, to query $x$ answer $\begin{cases} f(x) & \mbox{if } (x \in C_1) \vee \dots \vee  (x \in C_{k-1}) = \text{False} \\ +1 & \mbox{if } (x \in C_1) \vee \dots \vee  (x \in C_{k-1}) = \text{True and } \|x\|_2 \leq 1.15 \\ -1 & \mbox{if } (x \in C_1) \vee \dots \vee  (x \in C_{k-1}) = \text{True and } \|x\|_2 > 1.15 \end{cases}$ 
    \STATE {\bfseries Return} $p$
\end{algorithmic}
\end{algorithm}
At the first sight it might seem that the protocol for $\adv'$ uses $kq$ queries. But due to the fact that $k-1$ caps were added artificially the answer to $(k-1)q$ of those queries is known to $\adv'$ beforehand. This gives us that $\adv'$ is $q$-bounded as every query of $\adv'$ corresponds to a query of $\adv$. 

For simplicity we will refer to $C_i$'s and $C$ as caps even though they formally are caps with a ball carved out of them. Let $C \subseteq \R^d$ be the hidden cap that was generated from $\text{Cap}$. Observe that: 
\[C \cup \bigcup_{i=1}^{k-1} C_i\]
is distributed according to $\text{Cap}_k^{\text{i.i.d.}}$, as $C_1, \dots, C_{k-1}$ are i.i.d. uniformly random spherical caps, $C$ is a random spherical cap. Thus by the guarantee for $\adv$ we know that with probability at least $0.99$:
\[
\mu \left(p^{-1} \left(C \cup \bigcup_{i=1}^{k-1} C_i \right) \right) \geq \frac{1}{2} \cdot \mu \left( \left(C \cup \bigcup_{i=1}^{k-1} C_i\right) + B_{\e} \right)   
\]
As $C,C_1, \dots, C_k$ are indistinguishable from the point of view of $\adv$ we get that with probability at least $0.99/k$:
\[
\mu \left(p^{-1}(C) \right) \geq  \frac{1}{2k} \cdot \mu \left( \left(C \cup \bigcup_{i=1}^{k-1} C_i\right) + B_{\e} \right)     \text{,}
\]
where $\mu \left( \left(C \cup \bigcup_{i=1}^{k-1} C_i\right) + B_{\e} \right) \geq 1/4$ with probability $1 - 2^{-k}$ as with this probability there is a cap among $C,C_1, \dots, C_{k-1}$ that is of the form $text{cap}(y_i, 1.15, \tau(\delta/k))$. Then $text{cap}(y_i, 1.15, \tau(\delta/k)) + B_e$ covers $1/4$ of the mass of $\mu$. Thus by the union bound we get that with probability at least $0.99/k - 2^{-k} \geq 1/3k$:
\[\mu \left(p^{-1}(C) \right) \geq  \frac{1}{8k}  \text{,} \]
which is equivalent to $\adv'$ succeeding on $\text{Cap}(0.01/k)$ with approximation constant of $ \frac{1}{2k}$, error probability of at most $1 - \frac{1}{3k}$ for the same $\e$.
\end{proof}

In the next lemma we generalize Lemma~\ref{lem:reductioniid} to more complex distributions. More formally we show that if there is an adversary that uses $q$ queries and succeeds on $\text{Cap}_k^{\mathcal{G}}$ then there exists an adversary that succeeds on $\text{Cap}$ and asks at most $kq$ queries. Formally: 


\begin{lemma}[\textbf{Reduction from $\text{Caps}_k^{\mathcal{G}}$ to $\text{Cap}$}]
\label{lem:reductiongeneral}
Let $k\in \N_+$ and let $\mathcal{G}$ be any distribution on $(S_1^{d-1})^k$. If there exists a $q$-bounded adversary $\adv$ that succeeds on $\text{Caps}_k^{\mathcal{G}}(0.01)$ with approximation constant $1/2$, error probability $0.01$ and 
$\e = \tau(0.01/k)$ 
then there exists a $kq$-bounded adversary $\adv'$ that succeeds on $\text{Cap}(0.01/k)$ with approximation constant $\frac{1}{2k}$, error probability $0.76$ and the same $\e$.
\end{lemma}

\begin{proof}
Algorithm~\ref{alg:emulategeneral} defines the protocol for $\adv'$. We will show that this protocol satisfies the statement of the lemma.
\begin{algorithm}[H]
\caption{\textsc{EmulateGeneral}($f, \adv, \mathcal{G}, k$) \hfill $\triangleright$  $f$ is the attacked classifier
\newline \text{ } \hfill $\triangleright$  $\adv$ is an adversary for distribution $\text{Cap}_k^{\mathcal{G}}(0.01)$ }
\label{alg:emulategeneral}
\begin{algorithmic}[1]
    \STATE $T(x) := \begin{cases} 1.3 \cdot x & \mbox{if } \|x\|_2 \leq 1.15 \\ x/1.3 & \mbox{if } \|x\|_2 > 1.15 \end{cases}$
    \STATE $(y_1, \dots, y_k) \sim \mathcal{G}$
    \FOR{i = 1, \dots, k}
        \STATE $R_i :=$ rotation such that $R_i(e_1) = y_i$
        \hfill $\triangleright$ Any rotation satisfying the condition is valid
    \ENDFOR
    \STATE $M \sim O(d)$
    \STATE $b_1, \dots, b_{k} \sim U\{-1,1\}$
    \FOR{i = 1, \dots, k}
        \STATE $T_i := \begin{cases} T & \mbox{if } b_i = -1 \\ \text{Id} & \mbox{if } b_i = +1 \end{cases}$
    \ENDFOR
    \FOR{i = 1, \dots, k}
        \STATE $err_i := \left(f(M(R_i(T_i(x))) = -1 \wedge \|T_i(x)\|_2 > 1.15 \right) \vee \left(f(M(R_i(T_i(x))) = +1 \wedge \|T_i(x)\|_2 \leq 1.15\right)$
    \ENDFOR
    \STATE $err :=  \bigvee_{i \in [k]} err_i$
    \STATE $p := $ Simulate $\adv$, to $x$ answer 
    $\begin{cases} 
    +1 & \mbox{if } \|x\|_2 \leq 1.15 \wedge (err = \text{True})\\ 
    -1 & \mbox{if } \|x\|_2 \leq 1.15 \wedge (err = \text{False})\\ 
    -1 & \mbox{if } \|x\|_2 > 1.15 \wedge (err = \text{True})\\ 
    +1 & \mbox{if } \|x\|_2 > 1.15 \wedge (err = \text{False}) 
    \end{cases}$ \label{line:bigor}
    \FOR{i = 1, \dots, k}
        \STATE $p_i := T_i^{-1} \circ R_{i}^{-1} \circ M^{-1} \circ p \circ M \circ R_i \circ T_i$
    \ENDFOR
    \STATE {\bfseries Return} $p' := \frac{1}{k} \sum_{i=1}^{k} p_i \Big|_{S_1^{d-1}} + \text{Id }\Big|_{S_{1.3}^{d-1}} $ 
    \hfill $\triangleright$ understood as a linear combination of transport maps
\end{algorithmic}
\end{algorithm}

First observe that $\adv'$ asks at most $kq$ queries as every query of $\adv$ is multiplied $k$ times (see line~\ref{line:bigor} of Algorithm~\ref{alg:emulategeneral}). Observe that $p'$ is a well defined $\e$-perturbation as all $p_i$'s are $\e$-perturbations when restricted to $S_1^{d-1}$. It follows from the fact that all $p_i$'s are of the form $F^{-1} \circ p \circ F$ where  $F$ is a composition of an isometry and either $T$ or the identity. This implies that for all $x \in S_1^{d-1}$ we have $\|x - F^{-1} \circ p \circ F(x)\|_2 \leq \e$. Let $C$ be the hidden spherical cap. Observe that: 
\[ \bigcup_{i=1}^{k} M(R_i(T_i(C)))\]
is distributed according to $\text{Cap}_k^{\mathcal{G}}$, as the relative positions of normal vectors of $M(R_1(C)), M(R_2(C)), \dots, M(R_k(C))$ are distributed according to the process: generate $(y_1', \dots, y_k') \sim \mathcal{G}$, $M' \sim O(d)$, return $M'((y_1', \dots, y_k'))$. 
Thus by the fact that $\adv$ succeeds with $\alpha = 1/2$ we know that with probability at least $0.99$: 

\[\mu\left(p^{-1} \left(\bigcup_{i=1}^{k} M(R_i(T_i(C))) \right) \right) \geq \frac{1}{2} \cdot \mu \left( \left(\bigcup_{i=1}^{k} M(R_i(T_i(C)))\right) + B_{\e} \right) \text{.}\]
If $C \subseteq B_{1.5}$ then:
\[\mu \left(\left(\frac{1}{k} \sum_{i=1}^k p_i\right)^{-1}(C) \right) = \frac{1}{k} \sum_{i=1}^k \mu \left( p^{-1}(M(R_i(T_i(C)))) \right) \geq \frac{1}{k} \cdot \mu\left(p^{-1} \left(\bigcup_{i=1}^{k} M(R_i(T_i(C))) \right) \right)\]
Combining the two bounds we get that if $C \subseteq B_{1.5}$ then with probability at least $0.99$:
\begin{equation}\label{eq:pprimelwrbnd}
\mu \left(p'^{-1}(C) \right) \geq \frac{1}{2k} \cdot \mu \left( \left(\bigcup_{i=1}^{k} M(R_i(T_i(C)))\right) + B_{\e} \right)
\end{equation}
We note that with probability at least $(1 - 2^{-k}) \cdot 1/2$ we have that $C \subseteq B_{1.5}$ and there exists $i_0 \in [k]$ such that $T_{i_0}(C) \subseteq B_{1.5}$ as the two events are independent. This means that with probability at least $1/4$:
\begin{equation}\label{eq:advriskbndgeneral}
\mu \left( \left(\bigcup_{i=1}^{k} M(R_i(T_i(C)))\right) + B_{\e} \right) \geq 1/4 \text{,}
\end{equation}
as $\mu(M(R_{i_0}(T_{i_0}(C))) + B_\e) = \mu(S_1^{d-1})/2$.
Combining~\eqref{eq:pprimelwrbnd} and \eqref{eq:advriskbndgeneral} and using the union bound we get that with probability of at least $0.24$:
\[\mu(p'^{-1}(C)) \geq \frac{1}{8k} \text{,}\]
which is equivalent to $\adv'$ succeeding on $\text{Cap}(0.01/k)$ with approximation constant of at least $ \frac{1}{2k}$, error probability of at most $0.76$ for the same $\e$.

\end{proof}

The following tail bound will be useful.
\begin{lemma}\label{lemma:gaussianbounds}
Let $X$ be a zero-mean Gaussian with variance $\sigma^2$. Then for every $t \geq 0$:
\[\frac{1}{\sqrt{2 \pi}} \cdot \left( \frac{1}{t} - \frac{1}{t^3} \right) \cdot e^{-t^2/2} \leq \mathbb{P}_{X \sim \mathcal{N}(0,\sigma^2)}[X \geq \sigma \cdot t] \leq \frac{1}{\sqrt{2 \pi}} \cdot \frac{1}{t} \cdot e^{-t^2/2}\]
\end{lemma}

In  Lemma~\ref{lem:reductioniid} and Lemma~\ref{lem:reductiongeneral} we showed that the QC of $\text{Cap}_k^{\text{i.i.d.}}$ and $\text{Cap}_k^{\mathcal{G}}$ can be lower-bounded in terms of the QC of $\text{Cap}$. We now show an upper bound $\Theta(d)$ for the the QC of $\text{Cap}$. Further, we give the proof for a lower bound of $\Theta(d)$ for $\text{Cap}(0.01)$. The summary of these results is presented in Table~\ref{tab:sphereslwrbnds}.

The upper-bound for $\text{Cap}$, that we are going to show, holds even if we restrict the adversary to be non-adaptive. I.e., the bound holds even if we require the adversary to declare the set of queries up front.

\begin{definition}[\textbf{Non-adaptive query-bounded adversary}] \label{def:queryboundedadversarynonadapt}
For $\e \in \mathbb{R}_{\geq 0}$ and $f : \R^d \xrightarrow{} \{-1, 1\}$ a $q$-bounded adversary with parameter $\e$ is a deterministic algorithm $\adv$ that asks at most $q \in \N$ \textbf{non-adaptive} queries of the form $ f(x) \stackrel{?}{=} 1$ and outputs an $\e$-perturbation $\adv(f) : \R^d \xrightarrow{} \R^d$.
\end{definition}

\begin{lemma}[\textbf{Upper bound for $\text{Cap}$}]\label{lem:upperbnd}
For every $d$ bigger than an absolute constant there exists a non-adaptive $\Theta(d)$-bounded adversary $\adv$ that succeeds on $\text{Cap}(0.01)$ with approximation constant $1/2$, error probability $0.01$ for $\e = \tau(0.01)$.
Moreover $\adv$ can be implemented in $O(d^2)$ time.
\end{lemma}

\begin{proof}
We will first show that there exists a randomized $\adv$ that satisfies the statement of the Lemma. This adversary uses Algorithm~\ref{alg:upperbnd} invoked with $s = \Theta(d)$ as its protocol. Later we will show how to derandomize the protocol. The adversary we design is more produces adversarial examples only on the support of $\mathcal{D}$. This makes the goal of the adversary harder to achieve.

\paragraph{Randomized algorithm.}

\begin{algorithm}[H]
\caption{\textsc{CapAdversaryRandomized}($f, s, \e$) \hfill $\triangleright$ $f$ is the classifier, $s$ is the number of sampled points per sphere
\newline \text{ } \hfill $\triangleright$  $\e$ is the bound on allowed perturbations}
\label{alg:upperbnd}
\begin{algorithmic}[1]
    \STATE $Q^- := \left\{x_1^-, \dots, x_{s}^-\right\} \text{, where } x_i^-\text{'s are i.i.d.} \sim U[S_1^{d-1}]$ \label{line:gen:qenqmin}
    \STATE $Q^+ := \left\{x_1^+, \dots, x_{s}^+\right\} \text{, where } x_i^+\text{'s are i.i.d.} \sim U[S_{1.3}^{d-1}]$ \label{line:genqplus}
    \STATE $R := \left\{x \in Q^- : f(x) = +1\right\} \cup \left\{x \in Q^+ : f(x) = -1\right\}$ \label{line:askqueries}
    \STATE $v := 1/|R| \cdot  \sum_{x \in R} x$
    \STATE $p(x) := \begin{cases} \text{argsup}_{x' \in S_1^{d-1}, \|x - x'\|_2 \leq \e} \ \langle x' - x, v \rangle & \mbox{if } x \in S_1^{d-1} \\ \text{argsup}_{x' \in S_{1.3}^d, \|x - x'\|_2 \leq \e} \ \langle x' - x, v \rangle & \mbox{if $x \in S_{1.3}^{d-1} $} \end{cases} $
    \STATE {\bfseries Return} $p$
\end{algorithmic}
\end{algorithm}

First notice that $\adv$ is non-adaptive. The queries asked by $\adv$ are from $Q^- \cup Q^+$ which were generated (see lines~\ref{line:gen:qenqmin} and \ref{line:genqplus}) before any queries were asked and, hence, answered were received. 
Note further that $\adv$ is $\Theta(d)$ bounded as she asks $2 \cdot s = \Theta(d)$ queries. 

\paragraph{Run time.} We first remark that $\adv$ can be implemented in $O(d^2)$ time as the run time is dominated by asking $\Theta(d)$ queries and each vector is in $\R^d$. Formally, $p$ is not returned explicitly but one can imagine that $\adv$, after preprocessing that takes $O(d^2)$ time, provides oracle access to $p$ where each evaluation takes time $O(d)$. 

Now we prove that $\adv$ succeeds with probability $0.99$ with approximation constant $1/2$. Let $E$ be the hidden spherical cap that contains all errors of $1$-NN and let $u \in S_1^{d-1}$ be its normal vector. 
First assume that $E \subseteq S_1^{d-1}$. We start by lower-bounding $|R|$. 
For every $i \in [s]$ let $Y_i^-$ be a random variable which is equal to $1$ if $x_i^- \in E$ and $0$ otherwise. Then, by the Chernoff bound, we have that for every $\delta < 1$:
\begin{equation}\label{eq:twosided}
\mathbb{P}\left[ \left|\sum_{i = 1}^{s} Y_i^- - \mathbb{E}\left[\sum_{i=1}^{s} Y_i^-\right] \right| > \delta \cdot \mathbb{E}\left[\sum_{i=1}^{s} Y_i^-\right] \right] 
\leq
2e^{-\frac{\delta^2}{3} \mathbb{E}\left[\sum_{i=1}^{s} Y_i^-\right]}, 
\end{equation}
Noticing that $\mathbb{E}\left[\sum_{i=1}^{s} Y_i^-\right] = s \cdot 0.01$ if we set $\delta = 1/2$ we get that:
\[
\mathbb{P}\left[ \left|\sum_{i = 1}^{s} Y_i^- - s/100 \right| > s/200 \right] 
\leq
2e^{-\frac{\delta^2}{3} \cdot s/100} = 2e^{- s/1200}.
\]
So with probability at least $1 - 2e^{-s/1200}$ we have that:
\begin{equation}\label{eq:numberinElwrbnb}
|R| \geq s/200.
\end{equation}

Now assume $R = \{z_1, \dots, z_k\}$ and observe that for every $z \in R$ we have $\langle z, u \rangle \geq \tau(0.01)$ and note that $z_i$'s are i.i.d. uniformly distributed on $\text{cap}(u,1,\tau(0.01))$. We will model $U[S_1^{d-1}]$ as $\mathcal{N}(0,1/d)^d$. Then we have that:
\begin{align}
\langle u, v \rangle
&= 
\frac{1}{k} 
\left\langle u, \sum_{i = 1}^k z_i \right\rangle \nonumber \\
&= \frac{1}{k} \sum_{i = 1}^k \left\langle u,  z_i \right\rangle \nonumber \\
&\geq \tau(0.01) && \text{as $z_i \in R$} \nonumber \\
&\geq 2.2 / \sqrt{d} &&\text{by Lemma~\ref{lemma:gaussianbounds}} \label{eq:gooddirectioncorrelation}
\end{align}
Moreover if $\Pi$ is the orthogonal projection onto $u^{\perp}$ then $
\Pi(k \cdot v) \sim \mathcal{N}(0,k/d)^{d-1}$ \text{ and}
$
\Pi(v) \sim \mathcal{N}(0,1/(d k))^{d-1}$ thus:
\[
\|\Pi(v)\|_2^2 \sim \frac{1}{d k} \cdot \chi^2(d-1)
\]
So, using standard tail bounds for $\chi^2$ distribution, we get that for all $t \in (0,1)$:
\begin{equation}\label{eq:tailchisquared}
\mathbb{P} \left[\left| \frac{d k}{d-1} \cdot 
\|\Pi(v)\|_2^2 - 1 \right| \geq t \right] \leq 2e^{-(d-1)t^2/8}   
\end{equation}
Moreover observe:
\begin{align}
\left\langle u, \frac{v}{\|v\|_2} \right\rangle
&= \frac{\langle u, v \rangle}{\sqrt{\langle u, v \rangle^2 + \|\Pi(v)\|_2^2}} \nonumber \\
&= \frac{1}{\sqrt{1 + \|\Pi(v)\|_2^2 / \langle u, v \rangle^2}} \nonumber \\
&\geq \frac{1}{\sqrt{1 + \frac{d}{2.2^2} \cdot\|\Pi(v)\|_2^2  }}  &&\text{by~\eqref{eq:gooddirectioncorrelation}} \label{eq:dotprodlwrbnd}
\end{align}

Observe that if $\sphericalangle(u, v/\|v\|_2) = 0$ then $\mu(p^{-1}(\text{cap}(u,1,\tau(0.01)))) \geq 1/5$, as the preimage is exactly $(\text{cap}(u,1,\tau(0.01)) + B_e )\cap S_1^{d-1}$. Moreover $\mu(p^{-1}(\text{cap}(u,1,\tau(0.01))))$ is a continuous function of $\sphericalangle(u, v/\|v\|_2)$. Observe that in a coordinate system where the first basis vector is $v/\|v\|_2$ we have $p(\mu|_{S_1^{d-1}}) \approx (\mathcal{N}(\tau(0.01),1/d), \mathcal{N}(0,1/d), \dots, \mathcal{N}(0,1/d))$. Assume $\sphericalangle(u, v/\|v\|_2) = \alpha$. We bound:
\begin{align*}
&\mu(p^{-1}(\text{cap}(u,1,\tau(0.01)))) \\
&= \int_{-\infty}^{+\infty} \int_{\frac{\tau(0.01) - x_1 \cos(\alpha)}{\sin(\alpha)}}^{+\infty} d/2\pi \cdot e^{-\frac{d}{2}((x_1 - \tau(0.01))^2 + x_2^2)} \ dx_2 \ dx_1\\
&\geq \int_{-\infty}^{+\infty} \int_{\frac{2.4/\sqrt{d} - x_1 \cos(\alpha)}{\sin(\alpha)}}^{+\infty} d/2\pi \cdot e^{-\frac{d}{2}((x_1 - 2.2/\sqrt{d})^2 + x_2^2)} \ dx_2 \ dx_1 && \text{by Lemma~\ref{lemma:gaussianbounds}}\\
&= \int_{-\infty}^{+\infty} \int_{\frac{2.4 - x_1' \cos(\alpha)}{\sin(\alpha)}}^{+\infty} 1/2\pi \cdot e^{-\frac{1}{2}((x_1' - 2.2)^2 + x_2'^2)} \ dx_2' \ dx_1' && \text{$x_1' = x_1 \cdot \sqrt{d}, x_2' = x_2 \cdot \sqrt{d} $}\\ 
\end{align*}
This means that there exists $\alpha \in (0,\pi/2]$ (independent of $d$) such that for all $v$ such that $\sphericalangle(u, v/\|v\|_2) \leq \alpha$ we have $\mu(p^{-1}(\text{cap}(u,1,\tau(0.01)))) \geq 1/6$. Thus by \eqref{eq:dotprodlwrbnd} we get that there exists $\xi > 0$ such that if $\| \Pi(v) \|_2^2 \leq \xi/d$ then $\sphericalangle(u, v/\|v\|_2) \leq \alpha$ and in turn $\mu(p^{-1}(\text{cap}(u,1,\tau(0.01)))) \geq 1/6$.

Setting $k := \frac{2d}{\xi}, t:= 1/2$ in \eqref{eq:tailchisquared} we get that with probability at least $1 - e^{-(d-1)/32}$ we have: 
\[
\|\Pi(v)\|_2^2 \leq \xi/d \text{,}
\]
which in turn means that with probability at least $1 - e^{-(d-1)/32}$:
\begin{equation}\label{eq:recoveryguarantee} 
\mu(p^{-1}(\text{cap}(u,1,\tau(0.01)))) \geq 1/6 \text{.}
\end{equation}
Now combining \eqref{eq:numberinElwrbnb}, \eqref{eq:recoveryguarantee} and the union bound we get that if we set $s := \frac{400d}{\xi}$ then with probability at least $1 - 2e^{-s/200} - e^{-(d-1)/32} = 1 - 2e^{-2d/\xi}-e^{-(d-1)/32}$ we have:
\begin{equation}\label{eq:finalguaranteerand} 
\mu(p^{-1}(\text{cap}(u,1,\tau(0.01)))) \geq 1/6  \text{.}
\end{equation}
This probability is bigger than $0.99$ if $d$ is bigger than an absolute constant that depends on $\xi$. Observing that $\mu(E + B_\e) \geq 1/5$ we conclude that if $E \subseteq S_1^{d-1}$ then if $s = \Theta(d)$ then with probability $0.99$ $\adv$ succeeds on $\text{Cap}$ with approximation constant at least $1/2$. To finish the proof one notices that the case $E \subseteq S_{1.3}^{d-1}$ is analogous. The final constant hidden under $\Theta(d)$ for the number of samples is a maximum of constants for $S_1^{d-1}$ and $S_{1.3}^{d-1}$.

\paragraph{Deterministic algorithm.}
We know show how to derandomize Algorithm~\ref{alg:upperbnd} to design an adversary $\adv_{\text{det}}$. We observe that in Algorithm~\ref{alg:upperbnd} randomness was used only to generate query points $Q^-, Q^+$. Instead of generating the query points randomly we use fixed sets. We define the deterministic adversary, $\adv_{\text{det}}$, as:
\[
\adv_{\text{det}} := \textsc{CapAdversaryDeterministic}(\cdot, Q^-, Q^+) \text{,}
\]
for fixed (for a given $d$) sets $Q^-,Q^+$ that we define next. 

\begin{algorithm}[H]
\caption{\textsc{CapAdversaryDeterministic}($f, Q^-, Q^+, \e$) \newline \text{ } \hfill $\triangleright$ $f$ is the classifier, $Q^-,Q^+$ are query points on $S_1^{d-1}, S_{1.3}^{d-1}$ respectively \newline \text{ } \hfill $\triangleright$  $\e$ is the bound on allowed perturbations}
\label{alg:upperbnddet}
\begin{algorithmic}[1]
    \STATE $R := \left\{x \in Q^- : f(x) = +1\right\} \cup \left\{x \in Q^+ : f(x) = -1\right\}$
    \STATE $v := 1/|R| \cdot  \sum_{x \in R} x$
    \STATE $p(x) := \begin{cases} \text{argsup}_{x' \in S_1^{d-1}, \|x - x'\|_2 \leq \e} \ \langle x' - x, v \rangle & \mbox{if } x \in S_1^{d-1} \\ \text{argsup}_{x' \in S_{1.3}^d, \|x - x'\|_2 \leq \e} \ \langle x' - x, v \rangle & \mbox{if $x \in S_{1.3}^{d-1} $} \end{cases} $
    \STATE {\bfseries Return} $p$
\end{algorithmic}
\end{algorithm}

For $u \in S_1^{d-1}$ let: 
\[
f_u(x) := \begin{cases} -1 & \mbox{if } x \in S_1^{d-1} \setminus \text{cap}(u,1,\tau(0.01) \\ +1 & \mbox{otherwise} \end{cases} \text{.}
\]
We say that an adversary succeeds on $f_u$ if she, run for $f_u$, returns $p$ such that $\mu(p^{-1}(\text{cap}(u,1,\tau(0.01)))) \geq 1/8$. From \eqref{eq:finalguaranteerand} we know that for every $d \in \N_+$, for every $u \in S_1^{d-1}$:
\[
\mathbb{P}_{x_1^-, \dots, x_{400d/\xi}^- \sim U \left[S_1^{d-1} \right]} \left[\adv(f_u,400d/\xi,\e) \text{ succeeds} \right] \geq 1 - 2e^{-2d/\xi} - e^{-(d-1)/32}
\]
Thus we get that for every $d \in \N_+$ that:
\[
\mathbb{P}_{u, x_1^-, \dots, x_{400d/\xi}^- \sim U \left[S_1^{d-1} \right]} \left[\adv(f_u,400d/\xi,\e) \text{ succeeds} \right] \geq 1 - 2e^{-2d/\xi} - e^{-(d-1)/32}
\]
And finally, this means that for every $d \in \N_+$ there exists $Q^-_d \subseteq S_1^{d-1}, |Q^-_d| = 400d/\xi$ such that:
\[
\mathbb{P}_{u \sim U \left[ S_1^{d-1} \right]} \left[\adv_{\text{det}}(f_u,Q^-_d,\emptyset,\e) \text{ succeeds} \right] \geq 1 - 2e^{-2d/\xi} - e^{-(d-1)/32}   
\]
Thus, for $d$ bigger than an absolute constant we get that conditioned on $E \subseteq S_1^{d-1}$ $\adv_{\text{det}}$ run with $Q^- = Q^-_d$ succeeds with probability at least $0.99$ and asks $|Q^-_d| = \Theta(d)$ queries. Analogous argument shows that for every $d \in \N_+$ there exists $Q^+_d \subseteq S_{1.3}^{d-1}, |Q^+_d| = \Theta(d)$ such that the following holds. For every $d$ bigger than an absolute constant conditioned on $E \subseteq S_{1.3}^{d-1}$ $\adv_{\text{det}}$ run with $Q^+ = Q^+_d$ succeeds with probability at least $0.99$. Combining these two results we get that $\adv_{\text{det}}$ satisfies statement of the lemma.
\end{proof}

\begin{remark}
As we have seen in the proof of Lemma~\ref{lem:upperbnd} it was more natural to design an adversary that was randomized. We believe that allowing the adversary to use randomness would not change the results in a fundamental way.
\end{remark}

\lwrbndforcap*


\begin{proof}
To simplify computations we will sometimes approximate the uniform distribution on $S_1^{d-1}$ as a $d$-dimensional normal distribution: $\mathcal{N} \left(0,\frac{1}{d} \right)$. This change is valid as the norm of $\mathcal{N} \left(0,\frac{1}{d} \right)^d$ is closely concentrated around $1$.

\paragraph{Lower-bounding $QC$.} 
$\mathcal{A}$ succeeds on $\text{Cap}(0.01)$ with probability at least $1/3$ this means that if succeeds with probability at least $1/3$ on either $\text{Cap}(0.01)$ conditioned on the error set intersecting $S_1^{d-1}$ or $\text{Cap}(0.01)$ conditioned on the error set intersecting $S_{1.3}^{d-1}$. We first prove the result in the first case.

To use Theorem~\ref{thm:reduction} we think that there is an algorithm $ALG$ for which the distribution of errors coincides with $\text{Cap}(0.01)$ conditioned on the error set intersecting $S_1^{d-1}$. Let's call this distribution $Cap'(0.01)$. 
Note that by definition $AR(ALG(S),\e) = 1/2$. Thus we analyze:
\begin{align}
&\sup_{p : \text{ $\e$-perturbation}} \mathbb{P}_{S \sim \mathcal{D}^m} \left[ \mu(p^{-1}(E(ALG(S)))) \geq (1- \delta) \cdot AR(ALG(S), \e) \right] \nonumber \\
&= \sup_{p : \text{ $\e$-perturbation}} \mathbb{P}_{E \sim \text{Cap}'(0.01)} \left[ \mu(p^{-1}(E)) \geq \frac{1-\delta}{2} \right]   \label{eq:optimizationnew} \text{,} 
\end{align}

for a constant $\delta$ that will be fixed later. Let $p$ be an $\e$-perturbation and $y \in S_1^{d-1}$ be such that $\mu(p^{-1}(\text{cap}(y, 1.15, \tau(0.01)) \setminus B_{1.15/1.3})) \geq \frac{1-\delta}{2}$. We will show that for every $x \in S_1^{d-1}$ if $\sphericalangle(y,x) \in \left[\frac{49\pi}{100},\frac{51\pi}{100} \right]$ then $\mu(p^{-1}(\text{cap}(x, 1.15, \tau(0.01)) \setminus B_{1.15/1.3})) < \frac{1-\delta}{2}$. This will conclude the proof as then:  
\begin{align*}
\mathbb{P}_{E \sim \text{Cap}'(0.01)} \left[ \mu(p^{-1}(E)) \geq 
\frac{1-\delta}{2} \right] 
&\leq 2 \cdot \mu \left(\text{cap} \left(\cdot,1,\arccos \left(\frac{49 \pi}{100} \right) \right) \right) \\
&\leq 2^{-\Omega(d)} &&\text{By Lemma~\ref{lemma:gaussianbounds}}
\end{align*}
combined with \eqref{eq:optimizationnew} and Theorem~\ref{thm:reduction} gives the result.

Now let $x \in S_1^{d-1}$ be such that $\sphericalangle(y,x) \in \left[\frac{49\pi}{100},\frac{51\pi}{100} \right]$. To simplify notation let $C_x := \text{cap}(x, 1.15, \tau(0.01)) \setminus B_{1.15/1.3}, C_y := \text{cap}(y, 1.15, \tau(0.01)) \setminus B_{1.15/1.3}$. Now define: 
\[
I := \left\{z \in S_{1}^{d-1} \ \middle| \ d(z,C_y) \leq \e \wedge d(z,C_x) \leq \e \wedge d(z,C_x \cap C_y) > \e \right\} \text{,}
\]
where $d$ denotes the $\ell_2$ distance between sets. By Lemma~\ref{lemma:gaussianbounds} we have:
\begin{equation}\label{eq:thresholdbounds}
2.2/\sqrt{d} \leq \tau(0.01) \leq 2.4/\sqrt{d}
\end{equation}
Now observe that:
\begin{align} 
I 
&\supseteq \left\{ z \in S_1^{d-1} \ \middle| \ \langle z,y \rangle \geq 0 \wedge \langle z,x \rangle \geq 0 \wedge \left\langle z, \frac{x+y}{\|x+y\|_2} \right\rangle  < \frac{2.2}{\sqrt{d} \cdot \cos(\sphericalangle(y,x)/2)} - \frac{2.4}{\sqrt{d}} \right\} \nonumber \\
&\supseteq \left\{ z \in S_1^{d-1} \ \middle| \ \langle z,y \rangle \geq 0 \wedge \langle z,x \rangle \geq 0 \wedge \left\langle z, \frac{x+y}{\|x+y\|_2} \right\rangle  < \frac{1}{20\sqrt{d}} \right\} =: \widehat{I} \label{eq:Isupersetihat}
\end{align}
where in the first transition we used \eqref{eq:thresholdbounds} 
and in the second transition we used that $\sphericalangle(y,x) \in \left[\frac{49\pi}{100},\frac{51\pi}{100} \right]$. Note that $\mu \left(\widehat{I} \right)$ is minimized for $\sphericalangle(y,x) =\frac{51\pi}{100}$. Thus:
\begin{align}
&\mu \left(\widehat{I} \right) \nonumber \\
&\geq \int_{0}^{\infty} \int_{\tan(\pi/100) \cdot x_1}^{\infty} d/2\pi \cdot e^{-\frac{d}{2}(x_1^2 + x_2^2)} \cdot \mathds{1}\left[x_1 \cos \left(\frac{51\pi}{200}\right) + x_2 \sin\left(\frac{51\pi}{200}\right) < \frac{1}{20\sqrt{d}}\right] \ dx_2 \ dx_1 \nonumber \\
&= \int_{0}^{\infty} \int_{\tan(\pi/100) \cdot x_1}^{\infty} 1/2\pi \cdot e^{-\frac{1}{2}(x_1^2 + x_2^2)} \cdot \mathds{1}\left[x_1 \cos \left(\frac{51\pi}{200}\right) + x_2 \sin\left(\frac{51\pi}{200}\right) < \frac{1}{20}\right] \ dx_2 \ dx_1 \text{,} \label{eq:ihatlwrbnd}
\end{align}
where the first equality comes from integration by substitution. The integral from \eqref{eq:ihatlwrbnd} is positive, which means that there exists $\delta > 0$ such that $\mu(\widehat{I}) > \delta$. Combining that with \eqref{eq:Isupersetihat} we get that $\mu(I) > \delta$. Observe that by definition of $I$ for every $ z \in I$ we have that at most one of $p(z) \in C_x$, $p(z) \in C_y$ can be true. Thus, using the fact that $\mu(C_x + B_\e) = \mu(C_y + B_\e) = 1/2$, we get that:
\begin{equation}\label{eq:minimumsmall}
\min(\mu(p^{-1}(C_x)), \mu(p^{-1}(C_y))) < 1/2 - \delta/2 \text{.}
\end{equation}
This ends the proof as by assumption we know that $\mu(p^{-1}(C_y)) \geq 1/2 - \delta/2$, so by \eqref{eq:minimumsmall} we get that $\mu(p^{-1}(C_x)) < 1/2 - \delta/2$. The proof for the other case is analogous.
\end{proof}

Note that Lemma~\ref{lem:lwrbndforcap} is equivalent to the statement of Conjecture~\ref{con:onecap} for $k = 1$. 

\begin{conjecture}[\textbf{$\text{Cap}$ conjecture}]\label{con:onecap}
For every $k \in [d]$ if a $q$-bounded adversary $\adv$ succeeds on $\text{Cap}(0.01/k)$ with approximation constant  $ \geq \frac{1}{2k}$, error probability $ \leq 1 - \frac{1}{3k}$ for $\e$ such that $ \text{cap}(\cdot,1,\tau(0.01/k)) + B_{\e} = \text{cap}(\cdot,1,0) $. Then 
\[q \geq \Theta(d) \text{.}\]
\end{conjecture}

\newpage
\section{Figures}
\label{sec:figures}
In Figure~\ref{fig:decisionboundaryknnapp}, similar to Figure~\ref{fig:decisionboundaryknn}, we present visualizations of decision boundaries for $1$-NN. Each subfigure represents a random decision boundary for a different sample $S \sim \mathcal{D}^m$. The aim of these visualizations is to give an intuition for why Theorem~\ref{thm:knn} is true. 
\begin{figure}[H]
  \centering
  \includegraphics[width=0.85\textwidth]{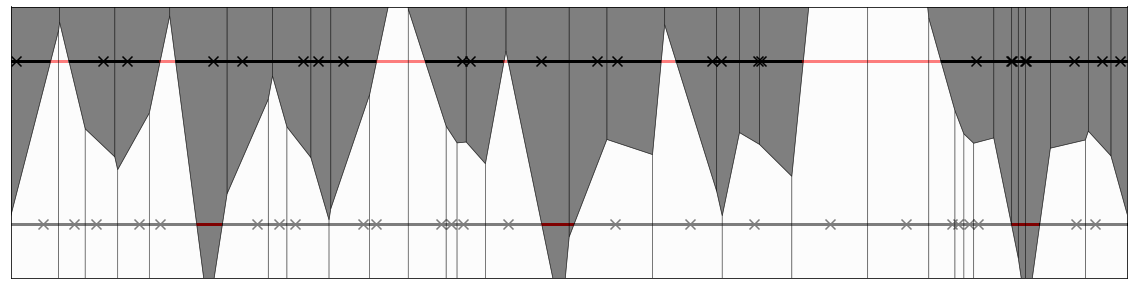}
  \includegraphics[width=0.85\textwidth]{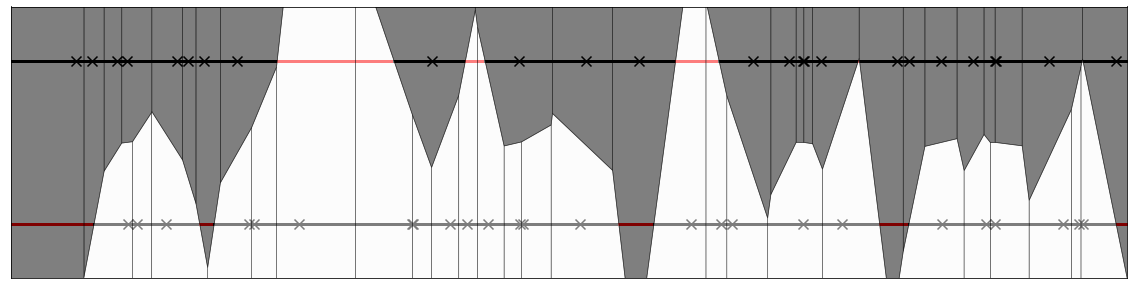}
  \includegraphics[width=0.85\textwidth]{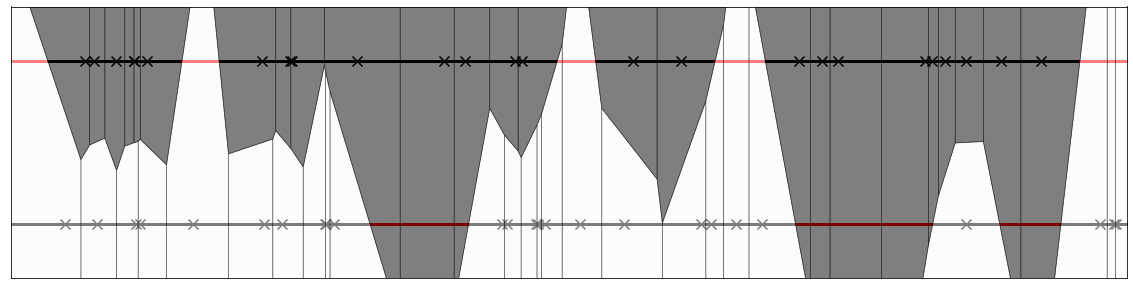}
  \includegraphics[width=0.85\textwidth]{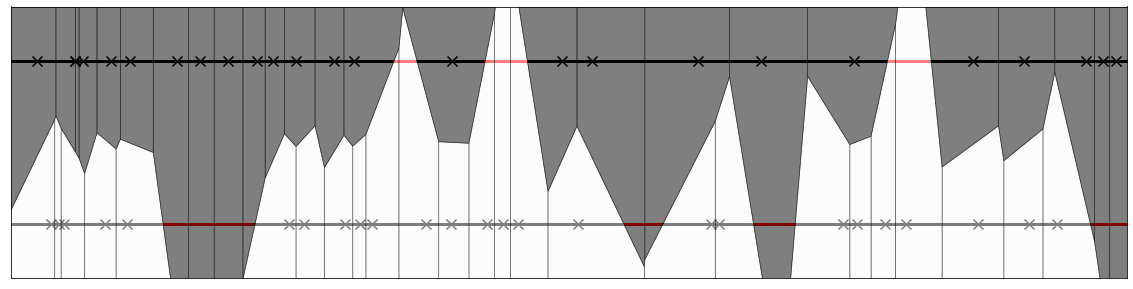}
  \includegraphics[width=0.85\textwidth]{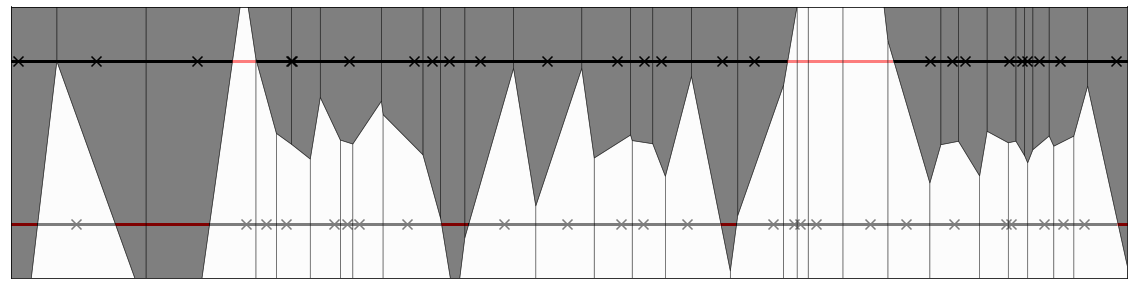}
  \caption{Random decision boundaries of $1$-NN for $T_{\text{intervals}}$.}
  \label{fig:decisionboundaryknnapp}
\end{figure}
\end{appendix}

\end{document}